\documentclass{article}
\usepackage{ijcai26}
\usepackage{booktabs}
\usepackage{tabularx}
\usepackage{array}
\usepackage{times}
\usepackage{soul}
\usepackage{url}
\usepackage[hidelinks]{hyperref}
\usepackage[utf8]{inputenc}
\usepackage[small]{caption}
\usepackage{subcaption}
\usepackage{graphicx}
\usepackage{amsmath}
\usepackage{amsthm}
\usepackage{amssymb}
\usepackage{algorithm}
\usepackage{algorithmic}
\usepackage[switch]{lineno}
\usepackage{multirow}
\usepackage{makecell}

\newtheorem{theorem}{Theorem}

\title{The Bell Boundary of Shared Representations}

\author{
Yang Wu \and
Hanyu Zhao \and
Yazhou Zhang\thanks{Corresponding author.} \and
Yuexian Hou
\\
\affiliations
Tianjin University, China\\
\emails
\{hanyzhao, yangwu\_, yxhou, yzhou\_zhang\}@tju.edu.cn
}

\begin{document}

\maketitle

\begin{abstract}   

When does a shared representation become insufficient for jointly supporting multiple tasks? This question lies at the heart of multi-task learning models. Although task interference in shared representations has been widely observed, its onset, measurable form, and capacity boundary remain poorly understood. In this paper, we revisit this problem through the lens of Bell's theorem. We show that if four binary task outcomes can be jointly represented by a common shared representation, their correlations must satisfy the classical CHSH bound $|S|\leq2$. We introduce NCnet, a controlled multi-task model for studying how representational capacity affects task compatibility. Experiments reveal a clear capacity-dependent transition. Once nontrivial task fitting emerges, constrained capacity can yield $|S|>2$. As capacity further increases, these exceedances progressively disappear and $S$ returns to the bounded region. This systematic pattern provides evidence that limited representational capacity is an important source of the observed incompatibility. We observe similar transitions across CNNs, PLMs, and LLMs, suggesting that the phenomenon extends beyond the controlled NCnet setting. We refer to this capacity-dependent transition as the Bell Boundary, providing a measurable diagnostic for representation bottlenecks, task incompatibility, and capacity allocation in multi-task models.

\end{abstract}

\begin{figure*}[t]
    \centering
    \includegraphics[width=\linewidth]{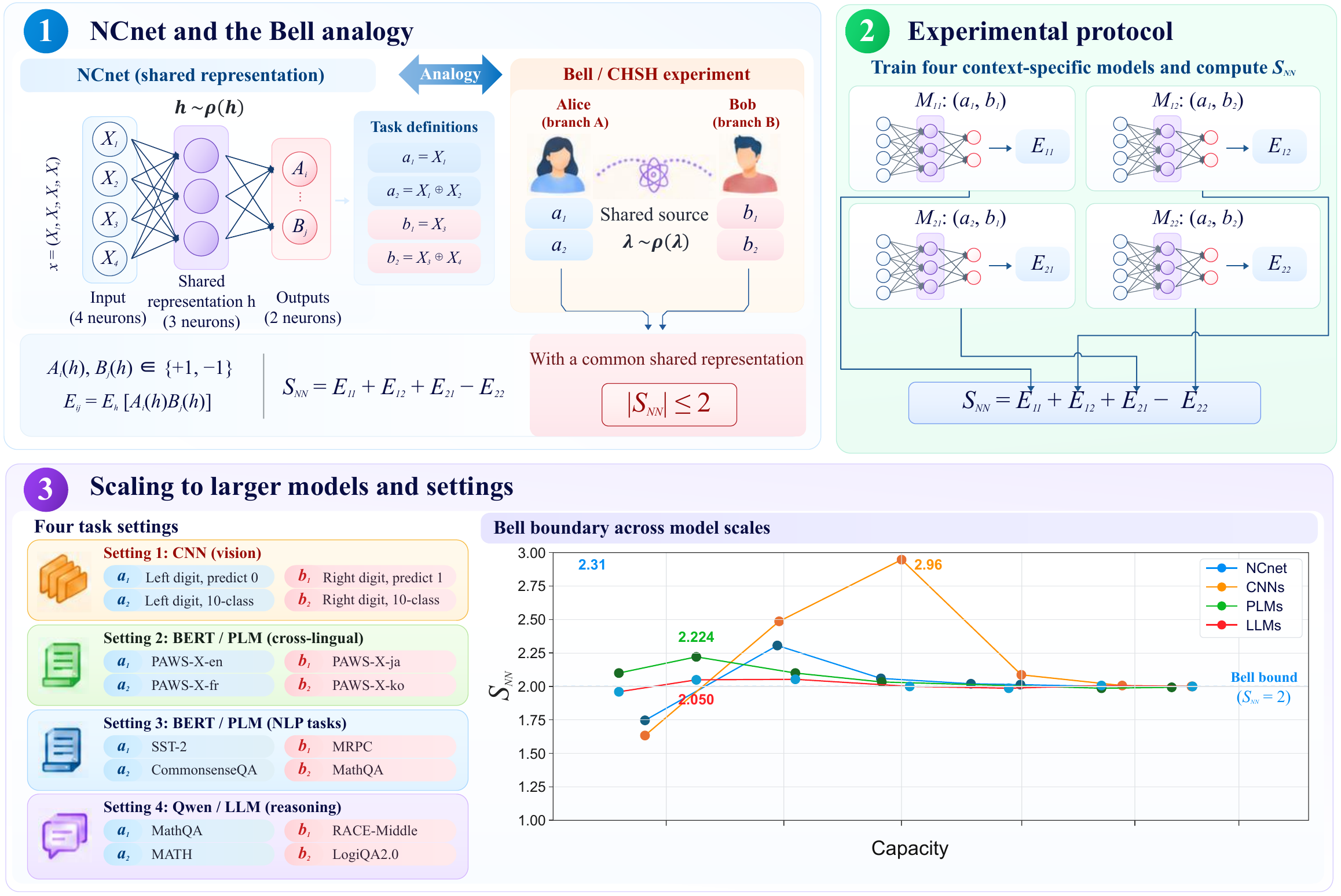}
    \caption{Overview of the Bell Boundary framework.
    \textbf{(1)} NCnet establishes a correspondence between shared-representation learning and the Bell/CHSH setting.
    \textbf{(2)} Four independently trained task-pair models provide the correlations used to compute $S_{\mathrm{NN}}$.
    \textbf{(3)} The analysis extends to CNNs, PLMs, and LLMs under four experimental settings. The bottom-right schematic illustrates CHSH exceedances in a critical capacity regime and their disappearance as capacity increases.}
    \label{fig:fig1}
\end{figure*}

\section{Introduction}
A key assumption underlying multi-task learning architecture is that multiple tasks can be supported by a shared latent representation~\cite{sener2018multi,ma2018modeling}. 
This assumption forms the foundation of multi-task foundation models, large language models (LLMs), and world models (WMs)~\cite{song2025large}, where a latent shared representation is expected to generalize across a broad range of downstream tasks, ranging from language understanding and generation~\cite{devlin2019bert,achiam2023gpt} to reasoning~\cite{song2025emotion}, planning~\cite{song2025large}, and action taking~\cite{song2025large}. 
Despite the remarkable success of this paradigm, a fundamental question remains unexplored:
\textit{When does a shared representation become insufficient for jointly supporting multiple tasks?}

Existing studies have long recognized that shared representations may induce task interference~\cite{sener2018multi,yu2020gradient}.
When representational capacity is limited, different tasks may impose conflicting optimization pressures on the shared parameters~\cite{standley2020tasks}, leading to gradient conflict~\cite{yu2020gradient}, negative transfer~\cite{ma2018modeling}, and unstable training~\cite{chen2018gradnorm}.
A large body of work has attempted to mitigate such effects through gradient balancing~\cite{chen2018gradnorm}, task routing~\cite{ma2018modeling}, modular architectures~\cite{ma2018modeling}, and mixture-of-experts models~\cite{ma2018modeling}. However, these methods mainly focus on reducing task conflict after it appears, instead of characterizing when and why a shared representation becomes insufficient. 
Three questions remain open: when does task conflict emerge, how can its strength be measured at the representation level, and whether there exists a principled boundary beyond which a single shared representation can no longer jointly support multiple tasks?

In this paper, we revisit this problem through the lens of Bell's theorem~\cite{Bell1964OnTE}. 
In the standard CHSH setting~\cite{clauser1969proposed}, one of the canonical Bell inequalities, two measurement branches $A$ and $B$ are considered. 
Branch $A$ has two possible settings $a_1$ and $a_2$, while branch $B$ has two possible settings $b_1$ and $b_2$.
Under a classical hidden-variable model, if the four potential binary outcomes can be jointly accounted for by a common hidden variable $\lambda \sim \rho \left ( \lambda  \right ) $, then their pairwise correlations satisfy the CHSH bound $|S|\leq 2$, where $S=E_{11}+E_{12}+E_{21}-E_{22}$ and $E_{ij}$ denotes the expectation of the product of the two binary outcomes for settings $(a_i,b_j)$, with $i,j\in\{1,2\}$. Thus, $|S|>2$ implies that the four outcomes cannot be embedded into a joint probability distribution~\cite{fine1982hidden} (i.e., no hidden-variable distribution \(\rho(\lambda)\) can jointly account for all four outcomes.).

This formulation naturally connects to shared-representation learning, as illustrated in Fig.~\ref{fig:fig1}.
In a multi-task neural network, a shared encoder maps an input $x\sim P(x)$ to a latent representation $h=f_\theta(x)$, inducing a representation distribution $h\sim\rho_\theta(h)$. 
We consider two task branches $A$ and $B$, each with two task settings $a_1$, $a_2$ and $b_1$, $b_2$. 
The corresponding binary task outcomes are denoted by $A_1(h)$, $A_2(h)$, $B_1(h)$, and $B_2(h)$, where each outcome indicates whether the prediction under the corresponding task setting is correct. 
We say that the four task contexts admit a common shared representation if their outcomes can be expressed as $A_i(h)$ and $B_j(h)$ on the same latent space $h\sim\rho(h)$, with each task outcome independent of its paired context.
Hereafter, we refer to such a context-independent common representation simply as a shared representation.
The central question is whether the four task-pair correlations can be jointly supported by a shared representation. We denote the CHSH statistic constructed from neural task-pair correlations by $S_{\mathrm{NN}}$. If such a representation exists, their correlations must satisfy the CHSH compatibility bound. Conversely, $|S_{\mathrm{NN}}|>2$ indicates that the four task-pair correlations cannot be jointly supported by a shared representation.

We interpret $S_{\mathrm{NN}}$ only when task performance exceeds chance level. Before this point, task outcomes are dominated by weak or noisy correlations, and $S_{\mathrm{NN}}$ mainly reflects underfitting rather than representation compatibility (see Sec.~\ref{sec:sec2.4}). Once nontrivial task fitting emerges, $|S_{\mathrm{NN}}|>2$ indicates that the learned task-pair correlations cannot be jointly supported by a single context-independent shared representation.

Hence, we introduce NCnet, a controlled multi-task architecture designed to align shared-representation learning with the CHSH setting. 
NCnet consists of a shared hidden layer and task-specific output heads. 
The two branches correspond to the two CHSH branches, and each branch contains two task settings. 
Each task output is converted into a binary correctness outcome, and the four pairwise task correlations are combined to compute the CHSH statistic $S_{\mathrm{NN}}$. By varying the number of hidden neurons, NCnet provides a controlled setting for studying how representational capacity affects task compatibility.

Experiments reveal a clear capacity-dependent transition. In low-capacity regimes, once the model achieves nontrivial task fitting, multiple tasks compete for the same representational dimensions, leading to strong gradient competition. 
The CHSH statistic can exceed the classical bound $|S_{\mathrm{NN}}|>2$
As capacity further increases, task competition weakens, task correlations become jointly compatible, and $S_{\mathrm{NN}}$ returns to the bounded region $|S_{\mathrm{NN}}|\le2$. We refer to this capacity-dependent transition as the \emph{Bell Boundary} of shared representations.

The significance of this boundary is not limited to the controlled NCnet setting. We observe consistent patterns in larger neural architectures, including CNN, PLMs and LLMs: homogeneous tasks are more easily supported by shared representations, whereas heterogeneous tasks are more likely to produce a boundary region with $S_{\mathrm{NN}}>2$. Increasing representational capacity consistently returns $S_{\mathrm{NN}}$ to the bounded region.
These results suggest that the Bell Boundary captures a general capacity transition in shared representations: before sufficient capacity is available, learned task correlations may not be jointly supported by a single shared representation; after sufficient capacity is reached, task interactions become compatible within the shared representation.

The Bell Boundary offers a quantitative tool for locating representation bottlenecks, measuring task incompatibility, and analyzing whether additional capacity is needed. More broadly, it provides theoretical guidance for the design of large multi-task foundation models.
Our contributions are summarized as follows:
\begin{itemize}
    \item We introduce the \emph{Bell Boundary} as a measurable capacity boundary for shared representations.
    \item We establish a formal correspondence between Bell inequalities and multi-task networks through NCnet.
    \item We reveal a capacity-dependent transition from gradient competition and $|S_{\mathrm{NN}}|>2$ to restored compatibility and $|S_{\mathrm{NN}}|\leq2$, and validate it from toy networks to PLMs and LLMs.
\end{itemize}

\section{Preliminaries}
\label{sec:prelim}


\subsection{Bell's Theorem and the CHSH Inequality}
\label{subsec:chsh}
Bell's theorem states that any theory satisfying local realism imposes strict constraints on the correlations among measurement outcomes. Here, \emph{realism} means that outcomes are determined by pre-existing variables, while \emph{locality} means that the outcome at one measurement branch does not depend on the measurement setting chosen at the other branch~\cite{Bell1964OnTE,brunner2014bell}. 
Bell inequalities therefore provide testable bounds on the correlations that can be generated by a local hidden-variable model. In this paper, we focus on the Clauser--Horne--Shimony--Holt (CHSH) inequality~\cite{clauser1969proposed}, because it provides a simple criterion for whether four binary outcomes can be jointly represented through a common hidden variable.

In the standard CHSH setting, there are two measurement branches, denoted by $A$ and $B$. Branch $A$ has two settings $a_1$ and $a_2$, and branch $B$ has two settings $b_1$ and $b_2$. 
Under a local hidden-variable model, all potential outcomes are determined by a common hidden variable $\lambda\sim\rho(\lambda)$. 
We write these outcomes as $A_i(\lambda)$, $B_j(\lambda)$, where $i,j\in\{1,2\}$. Locality requires that $A_i(\lambda)$ depends only on $a_i$ and $\lambda$, while $B_j(\lambda)$ depends only on $b_j$ and $\lambda$. Therefore, the four potential outcomes $A_1(\lambda)$, $A_2(\lambda)$, $B_1(\lambda)$, and $B_2(\lambda)$ are jointly defined and admit a common joint distribution $P(A_1,A_2,B_1,B_2)$~\cite{fine1982hidden}.
For each pair of settings, the correlation is the expectation of the product of the two outcomes:
\begin{equation}
\begin{aligned}
    E_{ij}=\int A_i(\lambda)B_j(\lambda)\rho(\lambda)d\lambda
    \end{aligned}
\end{equation}

The CHSH statistic is then defined as:
\begin{equation}
\begin{aligned}
S=E_{11}+E_{12}+E_{21}-E_{22}
 \end{aligned}
\end{equation}
Since all outcomes are binary, any local hidden-variable model satisfies the classical CHSH bound: $|S|\leq2$~\cite{clauser1969proposed}.
$|S|>2$ implies that the observed pairwise correlations cannot be represented by a single common joint probability distribution consistent with the local hidden-variable assumptions~\cite{fine1982hidden}.

\begin{table}[t]
\centering
\small
\caption{Correspondence between the CHSH setting and multi-task learning.}
\label{tab:tab1}
\renewcommand{\arraystretch}{1.15}
\begin{tabular}{p{0.45\linewidth} p{0.45\linewidth}}
\midrule
\textbf{CHSH setting} & \textbf{Multi-task learning} \\
\midrule
Hidden variable $\lambda\sim\rho(\lambda)$ & Shared representation $h=f_\theta(x)$, $h\sim\rho_\theta(h)$ \\
Measurement branches $A$ and $B$ & Task families $A$ and $B$ \\
Measurement settings $a_1$, $a_2$ and $b_1$, $b_2$ & Task settings $a_1$, $a_2$ and $b_1$, $b_2$ \\
Binary outcomes $A_i(\lambda)$, $B_j(\lambda)$ & Binary task outcomes $g_{\phi_{A_i}}\left ( h \right )$, $g_{\phi_{B_j}}\left ( h \right )$ \\
Correlation $E_{ij}$ & Task correlation $E_{ij}$ \\
CHSH statistic $S$ & Neural CHSH statistic $S_{\mathrm{NN}}$ \\
Classical bound $|S|\leq2$ & $|S_{\mathrm{NN}}|\leq2$ \\
Violation $|S|>2$ &  $|S_{\mathrm{NN}}|>2$ \\
\midrule
\end{tabular}
\end{table}

\subsection{Multi-Task Learning with Shared Representations}
\label{subsec:mtl}
Multi-task learning aims to train a single model to solve multiple tasks by sharing part of its parameters across tasks. A standard formulation consists of a shared encoder and a set of task-specific heads~\cite{sener2018multi,ma2018modeling}. Given an input $x\sim P(x)$, the shared encoder maps it into a hidden representation $h=f_\theta(x)$, where $\theta$ denotes the shared parameters. The hidden representation also induces a distribution over the latent space, denoted by $h\sim\rho_\theta(h)$. For each task $k$, a task-specific head $g_{\phi_k}$ produces the prediction:
\begin{equation}
\begin{aligned}
\hat{y}_k=g_{\phi_k}(h)
 \end{aligned}
\end{equation}
where $\phi_k$ denotes the task-specific parameters. 

The effectiveness of this formulation relies on the assumption that the shared representation $h$ contains sufficient information to support all task-specific heads. 
However, when representational capacity is limited, multiple tasks may compete for the same hidden dimensions~\cite{yu2020gradient,standley2020tasks,mao2024resource}. 

To quantify this competition, let $\mathbf{g}_k^{(t)}=\nabla_\theta\mathcal{L}_k^{(t)}$ and $\mathbf{g}_l^{(t)}=\nabla_\theta\mathcal{L}_l^{(t)}$ denote the gradients of the two task losses with respect to the shared trainable parameters $\theta$ at recorded training step $t$. For nonzero gradients, their gradient cosine similarity is
\begin{equation}
\label{eq:gradient_cosine}
c_{kl}^{(t)}=
\frac{\langle\mathbf{g}_k^{(t)},\mathbf{g}_l^{(t)}\rangle}
{\|\mathbf{g}_k^{(t)}\|_2\|\mathbf{g}_l^{(t)}\|_2}.
\end{equation}
Positive, zero, and negative values indicate aligned, orthogonal, and conflicting gradient directions, respectively~\cite{yu2020gradient}. The gradient conflict rate is the fraction of recorded steps with a negative gradient inner product:
\begin{equation}
\label{eq:gradient_conflict_rate}
R_{kl}=\frac{1}{T}\sum_{t=1}^{T}
\mathbf{1}\!\left[\langle\mathbf{g}_k^{(t)},\mathbf{g}_l^{(t)}\rangle<0\right],
\end{equation}
where $T$ is the number of recorded training steps and $\mathbf{1}[\cdot]$ is the indicator function. For nonzero gradients, the conflict criterion is equivalent to $c_{kl}^{(t)}<0$. Thus, cosine similarity measures directional alignment at each step, whereas the conflict rate measures how frequently conflict occurs over the recorded steps.



\subsection{A CHSH Formulation of Shared-Representation}
\label{subsec:bridge}
In the CHSH setting, the central question is whether four potential measurement outcomes can be jointly represented through a common hidden variable $\lambda$. In multi-task learning, the analogous question is whether multiple task outcomes can be jointly supported by a common hidden representation $h$. This correspondence is summarized in Table~\ref{tab:tab1}.

We consider two task families, denoted by $A$ and $B$. Each family contains two task settings, written as $a_1$, $a_2$ for family $A$ and $b_1$, $b_2$ for family $B$. Each task prediction is obtained from the shared representation $h=f_\theta(x)$ and converted into a binary correctness outcome, denoted by $A_i$ or $B_j$. Each outcome takes values in $\{-1,+1\}$, with $+1$ indicating a correct prediction and $-1$ an incorrect prediction.

\paragraph{Sufficient resources and independent convergence (SIC).}
The independent convergence (IC) assumption states that the task heads of a neural network can converge independently of one another, so that a multi-head network can be viewed as a collection of logically independent networks. Under IC, the converged task outcomes $A_i$ and $B_j$ depend only on the shared representation $h$ and their own task settings, and are independent of the joint training context $(i,j)$. When all four outcomes are jointly represented through $h\sim\rho_\theta(h)$, the resulting task-response model has the form of a local hidden-variable (LHV) model, with $h$ as the hidden variable.

A key observation is that IC does not generally hold. However, when the network has sufficient feature representation resources, it is reasonable to expect IC to hold. We therefore formulate the sufficient-resources-and-independent-convergence (SIC) assumption as $\mathrm{SIC}\equiv\mathrm{S}\land\mathrm{IC}$, where S denotes sufficient representation resources and IC denotes independent convergence.

\begin{table*}[t]
  \centering
  \caption{%
    NCnet results across hidden dimensions $d=1$--$20$
    (mean\,$\pm$\,std over 100 random seeds).
    $E_{ij}$ denotes task-pair correlation, $\mathrm{Acc}$ denotes per-head
    training accuracy, and $P_0$ is the fraction of runs with near-zero mean best-loss.
    \textbf{Bold} $S_{\mathrm{NN}}$ values significantly exceed $2$.%
  }
  \label{tab:hidden_dim_sweep_pairwise}
  \setlength{\tabcolsep}{5pt}
  \renewcommand{\arraystretch}{1.10}
  \resizebox{\linewidth}{!}{%
  \begin{tabular}{c ccc ccc ccc ccc c c}
    \toprule
    \multirow{2}{*}{$d$}
      & \multicolumn{3}{c}{$A_1B_1$}
      & \multicolumn{3}{c}{$A_1B_2$}
      & \multicolumn{3}{c}{$A_2B_1$}
      & \multicolumn{3}{c}{$A_2B_2$}
      & \multirow{2}{*}{$S_{\mathrm{NN}}$}
      & \multirow{2}{*}{$P_0$ (\%)} \\
    \cmidrule(lr){2-4}\cmidrule(lr){5-7}\cmidrule(lr){8-10}\cmidrule(lr){11-13}
      & $E$ & $\mathrm{Acc}(A_i)$ & $\mathrm{Acc}(B_j)$
      & $E$ & $\mathrm{Acc}(A_i)$ & $\mathrm{Acc}(B_j)$
      & $E$ & $\mathrm{Acc}(A_i)$ & $\mathrm{Acc}(B_j)$
      & $E$ & $\mathrm{Acc}(A_i)$ & $\mathrm{Acc}(B_j)$
      & & \\
    \midrule
    1
      & $\phantom{-}0.00${\,\small$\pm 0.00$} & $0.72${\,\small$\pm 0.25$} & $0.78${\,\small$\pm 0.25$}
      & $\phantom{-}0.02${\,\small$\pm 0.07$} & $0.97${\,\small$\pm 0.10$} & $0.52${\,\small$\pm 0.07$}
      & $\phantom{-}0.01${\,\small$\pm 0.06$} & $0.52${\,\small$\pm 0.06$} & $0.98${\,\small$\pm 0.09$}
      & $\phantom{-}0.06${\,\small$\pm 0.06$} & $0.62${\,\small$\pm 0.11$} & $0.64${\,\small$\pm 0.11$}
      & $-0.02${\,\small$\pm 0.12$}
      & $0$ \\
    2
      & $\phantom{-}1.00${\,\small$\pm 0.00$} & $1.00${\,\small$\pm 0.00$} & $1.00${\,\small$\pm 0.00$}
      & $\phantom{-}0.48${\,\small$\pm 0.11$} & $1.00${\,\small$\pm 0.00$} & $0.74${\,\small$\pm 0.05$}
      & $\phantom{-}0.47${\,\small$\pm 0.10$} & $0.74${\,\small$\pm 0.05$} & $1.00${\,\small$\pm 0.00$}
      & $\phantom{-}0.22${\,\small$\pm 0.20$} & $0.76${\,\small$\pm 0.17$} & $0.76${\,\small$\pm 0.17$}
      & $\phantom{-}1.74${\,\small$\pm 0.24$}
      & $0$ \\
    3
      & $\phantom{-}1.00${\,\small$\pm 0.00$} & $1.00${\,\small$\pm 0.00$} & $1.00${\,\small$\pm 0.00$}
      & $\phantom{-}0.91${\,\small$\pm 0.21$} & $1.00${\,\small$\pm 0.00$} & $0.95${\,\small$\pm 0.10$}
      & $\phantom{-}0.93${\,\small$\pm 0.17$} & $0.96${\,\small$\pm 0.09$} & $1.00${\,\small$\pm 0.00$}
      & $\phantom{-}0.53${\,\small$\pm 0.24$} & $0.88${\,\small$\pm 0.14$} & $0.86${\,\small$\pm 0.15$}
      & $\mathbf{\phantom{-}2.31}${\,\small$\pm 0.35$}
      & $0$ \\
    \cmidrule(lr){1-15}
    4
      & $\phantom{-}1.00${\,\small$\pm 0.00$} & $1.00${\,\small$\pm 0.00$} & $1.00${\,\small$\pm 0.00$}
      & $\phantom{-}0.99${\,\small$\pm 0.07$} & $1.00${\,\small$\pm 0.00$} & $0.99${\,\small$\pm 0.04$}
      & $\phantom{-}0.98${\,\small$\pm 0.10$} & $0.99${\,\small$\pm 0.05$} & $1.00${\,\small$\pm 0.00$}
      & $\phantom{-}0.93${\,\small$\pm 0.11$} & $0.98${\,\small$\pm 0.04$} & $0.98${\,\small$\pm 0.04$}
      & $\mathbf{\phantom{-}2.04}${\,\small$\pm 0.16$}
      & $53$ \\
    5
      & $\phantom{-}1.00${\,\small$\pm 0.00$} & $1.00${\,\small$\pm 0.00$} & $1.00${\,\small$\pm 0.00$}
      & $\phantom{-}1.00${\,\small$\pm 0.00$} & $1.00${\,\small$\pm 0.00$} & $1.00${\,\small$\pm 0.00$}
      & $\phantom{-}1.00${\,\small$\pm 0.00$} & $1.00${\,\small$\pm 0.00$} & $1.00${\,\small$\pm 0.00$}
      & $\phantom{-}0.99${\,\small$\pm 0.05$} & $1.00${\,\small$\pm 0.02$} & $1.00${\,\small$\pm 0.02$}
      & $\mathbf{\phantom{-}2.01}${\,\small$\pm 0.05$}
      & $95$ \\
    6
      & $\phantom{-}1.00${\,\small$\pm 0.00$} & $1.00${\,\small$\pm 0.00$} & $1.00${\,\small$\pm 0.00$}
      & $\phantom{-}1.00${\,\small$\pm 0.00$} & $1.00${\,\small$\pm 0.00$} & $1.00${\,\small$\pm 0.00$}
      & $\phantom{-}1.00${\,\small$\pm 0.00$} & $1.00${\,\small$\pm 0.00$} & $1.00${\,\small$\pm 0.00$}
      & $\phantom{-}1.00${\,\small$\pm 0.01$} & $1.00${\,\small$\pm 0.00$} & $1.00${\,\small$\pm 0.01$}
      & $\phantom{-}2.00${\,\small$\pm 0.01$}
      & $99$ \\
    \cmidrule(lr){1-15}
    $7$--$20$
      & $\phantom{-}1.00${\,\small$\pm 0.00$} & $1.00${\,\small$\pm 0.00$} & $1.00${\,\small$\pm 0.00$}
      & $\phantom{-}1.00${\,\small$\pm 0.00$} & $1.00${\,\small$\pm 0.00$} & $1.00${\,\small$\pm 0.00$}
      & $\phantom{-}1.00${\,\small$\pm 0.00$} & $1.00${\,\small$\pm 0.00$} & $1.00${\,\small$\pm 0.00$}
      & $\phantom{-}1.00${\,\small$\pm 0.00$} & $1.00${\,\small$\pm 0.00$} & $1.00${\,\small$\pm 0.00$}
      & $\phantom{-}2.00${\,\small$\pm 0.00$}
      & $\mathbf{100}$ \\
    \bottomrule
  \end{tabular}%
  }
\end{table*}

Under the task-specific decomposition assumed by SIC, let $\vartheta_{A_i}$ and $\vartheta_{B_j}$ denote the effective parameter blocks for the two active tasks. Their joint loss separates as
\begin{equation}
\begin{aligned}
\mathcal{L}
&=\mathcal{L}_{A_i}(\vartheta_{A_i})
+\mathcal{L}_{B_j}(\vartheta_{B_j}),\\
\nabla_{\vartheta_{A_i}}\mathcal{L}_{B_j}
&=0,\qquad
\nabla_{\vartheta_{B_j}}\mathcal{L}_{A_i}=0.
\end{aligned}
\end{equation}
Thus, each task's optimization depends only on its own parameter block, and the two minimization problems decouple. Different minimizers may nevertheless implement different functions; IC explicitly requires functional equivalence across converged runs. Consequently, replacing the partner of task $a_i$ leaves its converged prediction function unchanged, and the same holds for task $b_j$, almost surely under the shared-representation distribution $\rho_\theta(h)$. With the same labeling rule for each task across pairs, the correctness outcomes are also unchanged. We therefore write the effective outcomes as $A_i(h)$ and $B_j(h)$ without context indices. This requires equality of task functions, not of network parameters. Optimization independence does not require statistical independence of the outcomes, which may remain correlated through the shared representation.

Taking the shared representation $h\sim\rho_\theta(h)$ as the hidden variable, for each pair $(a_i,b_j)$, define
\begin{equation}
E_{ij}=\mathbb{E}_{h\sim\rho_\theta(h)}[A_i(h)B_j(h)],
\end{equation}
and $S_{\mathrm{NN}}=E_{11}+E_{12}+E_{21}-E_{22}$.

\begin{theorem}[Neural CHSH bound under SIC]
\label{thm:sic_chsh}
Suppose the four task-pair models satisfy SIC: the networks have sufficient feature representation resources (S), and their task heads can converge independently of one another, with the converged task functions independent of the joint training context (IC). If the four correctness outcomes can be jointly represented as $A_i(h)$ and $B_j(h)$ through a shared representation $h\sim\rho_\theta(h)$, with the same labeling rule for each task across pairs, then $|S_{\mathrm{NN}}|\leq2$.
\end{theorem}

\begin{proof}
The four context-independent outcomes are functions of the shared representation $h$ on the same probability space. For $\rho_\theta$-almost every $h$,
\begin{equation}
\begin{aligned}
C(h)
&=A_1B_1+A_1B_2+A_2B_1-A_2B_2\\
&=A_1(B_1+B_2)+A_2(B_1-B_2),
\end{aligned}
\end{equation}
where the argument $h$ is omitted on the right-hand side. Since $B_1,B_2\in\{-1,+1\}$, exactly one of $B_1+B_2$ and $B_1-B_2$ is zero, while the other equals $\pm2$. Hence $|C(h)|=2$, and
\begin{equation}
|S_{\mathrm{NN}}|
=|\mathbb{E}_{h\sim\rho_\theta(h)}[C(h)]|
\leq\mathbb{E}_{h\sim\rho_\theta(h)}[|C(h)|]=2.
\end{equation}
\end{proof}

The observed violation of the neural CHSH inequality does not arise from quantum effects and involves no nonlocal correlations in the physical sense. It only rules out a joint probability distribution $P(A_1,A_2,B_1,B_2)$ consistent with the observed pairwise correlations and indicates the failure of the SIC assumption, namely $\neg\mathrm{SIC}\equiv\neg\mathrm{S}\lor\neg\mathrm{IC}$. To minimize the influence of random initialization and differences in optimization trajectories across individual runs, we repeat each experimental configuration with different random seeds and report the mean and standard deviation of the resulting statistics. Across our extensive experiments, we observe that the violation disappears as shared representational resources increase. This systematic capacity-dependent pattern further supports insufficient representational capacity as one of the important causes of the observed inequality violation.

The neural CHSH bound therefore provides a criterion for detecting SIC failure at a given representational capacity, although satisfying the bound does not establish SIC. We refer to the observed return from the CHSH-exceeding regime $|S_{\mathrm{NN}}|>2$ to the non-exceeding regime $|S_{\mathrm{NN}}|\leq2$ with increasing representational capacity as the \emph{Bell Boundary} of shared representations.

\subsection{When is $S_{\mathrm{NN}}$ Informative?}\label{sec:sec2.4}
Let $p_i$ and $q_j$ denote the accuracies of tasks $a_i$ and $b_j$, respectively.
Since $g_{\phi_{A_i}}(h)=+1$ for a correct prediction and $-1$ otherwise, we have $\mathbb{E}[g_{\phi_{A_i}}(h)]=p_i-(1-p_i)=2p_i-1$, and similarly $\mathbb{E}[g_{\phi_{B_j}}(h)]=2q_j-1$.

When a model performs at or below chance-level performance on both tasks, $p_i\approx q_j\approx0.5$, so $\mathbb{E}[g_{\phi_{A_i}}(h)]\approx0$ and $\mathbb{E}[g_{\phi_{B_j}}(h)]\approx0$.
If the prediction outcomes of the two tasks are also independent, then:
\begin{equation}
\begin{aligned}
E_{ij}&=\mathbb{E}_{h\sim\rho_\theta(h)}[g_{\phi_{A_i}}(h)g_{\phi_{B_j}}(h)]\\&\approx\mathbb{E}[g_{\phi_{A_i}}(h)]\mathbb{E}[g_{\phi_{B_j}}(h)]\approx0
 \end{aligned}
\end{equation}
and hence $S_{\mathrm{NN}}\approx0$. Therefore, in this case, $S_{\mathrm{NN}}<2$ should not be interpreted as evidence of representation compatibility. It only means that the task heads have not yet extracted task information from the shared representation

After the model exceeds chance-level performance, $S_{\mathrm{NN}}$ becomes meaningful as a representation-compatibility statistic. 
In this regime, $|S_{\mathrm{NN}}|>2$ can arise when the four task correlations become asymmetric. As capacity further increases, all tasks can be fitted more consistently. In the ideal case, $g_{\phi_{A_1}}(h)\approx g_{\phi_{A_2}}(h)\approx g_{\phi_{B_1}}(h)\approx g_{\phi_{B_2}}(h)\approx+1$, so $E_{11}\approx E_{12}\approx E_{21}\approx E_{22}\approx+1$ and $S_{\mathrm{NN}}\approx1+1+1-1=2$. Thus, the \textit{Bell Boundary} marks the point at which the model has moved beyond chance-level prediction, but shared representations of the given capacity are still insufficient to support all four task-pair combinations separately.

A subtlety deserves clarification. If a single feedforward network shared one representation $h$ among all four heads, then $A_1(h),A_2(h),B_1(h),B_2(h)$ would be deterministic functions of the same $h$ and would admit a joint distribution induced by $P(x)$; in this case $|S_{\mathrm{NN}}|\le2$ holds as an algebraic identity and cannot be exceeded. We therefore do not claim that a single network violates the bound. Instead, we adopt a pairwise training protocol: for each task pair $(a_i,b_j)$ we train an independent model $M_{ij}$ and estimate $E_{ij}$, then combine the four correlations into $S_{\mathrm{NN}}$ to test whether representations of the given capacity can support all four task-pair combinations separately. This protocol is detailed in Sec.~\ref{sec:ncnet_train}.

\begin{table*}[!t]
\centering
\caption{Task assignments for each experimental setting.}
\label{tab:task_combinations}
\begin{tabularx}{\textwidth}{@{} l *{4}{>{\centering\arraybackslash}X} @{}}
\toprule
Setting & $a_1$ & $a_2$ & $b_1$ & $b_2$ \\
\midrule
Setting 1              & Left digit, predict~0 & Left digit, 10-class & Right digit, predict~1 & Right digit, 10-class \\[3pt]
Setting 2           & PAWS-X-en             & PAWS-X-fr            & PAWS-X-ja              & PAWS-X-ko            \\[3pt]
Setting 3           & SST-2                 & CommonsenseQA        & MRPC                   & MathQA               \\
Setting 4           & MathQA                & MATH                 & RACE-Middle            & LogiQA2.0            \\
\bottomrule
\end{tabularx}
\end{table*}

\section{Related Work}
\subsection{Task Conflict in Multi-Task Learning}
Multi-task learning is built on the premise that related tasks can benefit from shared representations. In practice, however, parameter sharing can also induce negative transfer when task objectives are poorly aligned or when shared capacity is insufficient. Prior work has studied this issue from the perspective of multi-objective optimization and gradient interference. Sener and Koltun \cite{sener2018multi} formulate multi-task learning as a multi-objective optimization problem, where different task losses may impose competing optimization directions. Yu et al. \cite{yu2020gradient} further identify detrimental gradient interference as a key obstacle in multi-task optimization. These studies show that task conflict is not merely an implementation artifact, but a structural difficulty in shared-parameter learning.

Most existing methods address this difficulty by modifying the training dynamics or the sharing structure. GradNorm \cite{chen2018gradnorm} balances task training rates by adapting gradient magnitudes. PCGrad \cite{yu2020gradient} projects conflicting task gradients to reduce destructive interference. CAGrad \cite{liu2021conflict} constructs conflict-averse update directions that improve the average objective while controlling worst-case task degradation. Other approaches relax hard parameter sharing through routing, modularization, or expert-based architectures, such as multi-gate mixture-of-experts \cite{ma2018modeling}. These methods are effective for mitigating conflict, but they mainly operate after conflict has already appeared.

Our work is complementary to this line of research. Instead of proposing another conflict-mitigation algorithm, we ask a more diagnostic question: when does a shared representation become insufficient for jointly supporting multiple tasks? Existing methods typically measure conflict through gradients, losses, or downstream performance. We introduce a CHSH-based statistic to characterize task compatibility at the representation level, allowing us to identify a measurable capacity boundary of shared representations.

\subsection{Bell Inequalities and Machine Learning}
Bell inequalities have been studied in machine learning primarily in the context of quantum information and nonlocal correlations. Deng et al. \cite{deng2018machine} use machine learning to detect Bell nonlocality in quantum many-body systems. Ma and Yung \cite{ma2018transforming} combine Bell-inequality information with machine learning to classify quantum states and detect entanglement. Canabarro et al. \cite{canabarro2019machine} apply machine-learning methods to detect and quantify nonlocal correlations when direct Bell-inequality analysis becomes computationally difficult. Kriváchy et al. \cite{krivachy2020neural} introduce a neural-network oracle for nonlocality problems in quantum networks. In these studies, neural networks are mainly used as computational tools for analyzing quantum correlations or approximating classical and quantum strategies.

Different from them, we do not use neural networks to study quantum nonlocality, nor do we claim that classical neural networks exhibit quantum behavior. Instead, we use the CHSH inequality as a probabilistic criterion for studying task interference, representation bottlenecks, and capacity transitions in multi-task neural networks.

\begin{figure}[!t]
    \centering
    \includegraphics[width=\columnwidth]{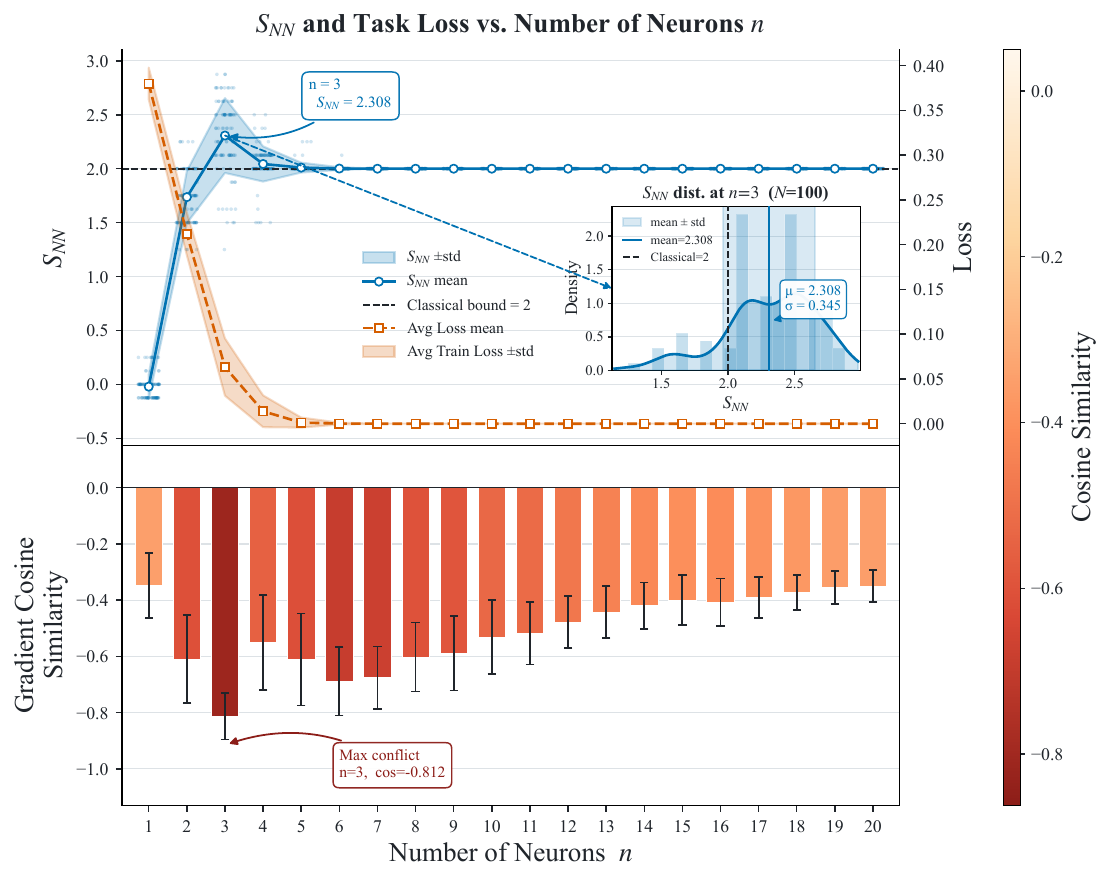}
    \caption{Capacity boundary and gradient conflict in NCnet. Top: $S_{\mathrm{NN}}$ (left axis) and average training loss (right axis) versus hidden width $n$. Points and bands show individual seeds and mean\,$\pm$\,std; the dashed line marks $S_{\mathrm{NN}}=2$, and the inset shows its distribution at $n=3$. Bottom: gradient cosine similarity averaged over the four $M_{ij}$ contexts and all training phases (green: alignment; red: conflict).}
    \label{img:ncnet-exp}
\end{figure}

\begin{table*}[t]
  \centering
  \caption{Train-phase $S_{\mathrm{NN}}$ (mean $\pm$ std over 5 random seeds) for the
    multi-task CNN across shared-encoder depth (rows) and channel width (columns).
    \textbf{Bold} cells are
    significantly greater than the classical Bell bound $S=2$}
  \label{tab:cnn_snn_train}
  \setlength{\tabcolsep}{5pt}
  \renewcommand{\arraystretch}{1.15}
  \resizebox{\linewidth}{!}{%
  \begin{tabular}{ccccccccc}
    \toprule
    \multirow{2}{*}{Depth} & \multicolumn{8}{c}{Channel width} \\
    \cmidrule(lr){2-9}
      & $1$ & $2$ & $4$ & $8$ & $16$ & $32$ & $64$ & $96$ \\
    \midrule
    1 & $-1.02${\,\small$\pm 0.57$} & $-1.82${\,\small$\pm 0.70$} & $-1.30${\,\small$\pm 0.04$} & $-1.23${\,\small$\pm 0.04$} & $-1.13${\,\small$\pm 0.03$} & $-0.95${\,\small$\pm 0.03$} & $-0.80${\,\small$\pm 0.02$} & $-0.76${\,\small$\pm 0.01$} \\
    2 & $-1.56${\,\small$\pm 0.50$} & $-1.22${\,\small$\pm 0.03$} & $-1.08${\,\small$\pm 0.04$} & $-0.85${\,\small$\pm 0.07$} & $-0.56${\,\small$\pm 0.01$} & $-0.32${\,\small$\pm 0.02$} & $-0.07${\,\small$\pm 0.05$} & $\phantom{-}0.07${\,\small$\pm 0.04$} \\
    3 & $-1.18${\,\small$\pm 0.03$} & $-1.01${\,\small$\pm 0.09$} & $-0.68${\,\small$\pm 0.05$} & $-0.13${\,\small$\pm 0.08$} & $\phantom{-}0.62${\,\small$\pm 0.06$} & $\phantom{-}1.41${\,\small$\pm 0.04$} & $\phantom{-}1.93${\,\small$\pm 0.04$} & $\mathbf{\phantom{-}2.26}${\,\small$\pm 0.07$} \\
    4 & $-0.89${\,\small$\pm 0.05$} & $-0.26${\,\small$\pm 0.11$} & $\phantom{-}0.68${\,\small$\pm 0.09$} & $\phantom{-}1.61${\,\small$\pm 0.04$} & $\mathbf{\phantom{-}2.46}${\,\small$\pm 0.07$} & $\mathbf{\phantom{-}2.96}${\,\small$\pm 0.04$} & $\mathbf{\phantom{-}2.06}${\,\small$\pm 0.02$} & $\phantom{-}2.00${\,\small$\pm 0.00$} \\
    5 & $-0.08${\,\small$\pm 0.09$} & $\phantom{-}1.41${\,\small$\pm 0.08$} & $\mathbf{\phantom{-}2.64}${\,\small$\pm 0.06$} & $\mathbf{\phantom{-}2.17}${\,\small$\pm 0.02$} & $\phantom{-}2.00${\,\small$\pm 0.00$} & $\phantom{-}2.00${\,\small$\pm 0.00$} & $\phantom{-}2.00${\,\small$\pm 0.00$} & $\phantom{-}2.00${\,\small$\pm 0.00$} \\
    6 & $\phantom{-}1.15${\,\small$\pm 0.18$} & $\phantom{-}2.01${\,\small$\pm 0.08$} & $\mathbf{\phantom{-}2.01}${\,\small$\pm 0.01$} & $\phantom{-}2.00${\,\small$\pm 0.00$} & $\phantom{-}2.00${\,\small$\pm 0.00$} & $\phantom{-}2.00${\,\small$\pm 0.00$} & $\phantom{-}2.00${\,\small$\pm 0.00$} & $\phantom{-}2.00${\,\small$\pm 0.00$} \\
    7 & $\phantom{-}1.77${\,\small$\pm 0.11$} & $\mathbf{\phantom{-}2.02}${\,\small$\pm 0.02$} & $\phantom{-}2.00${\,\small$\pm 0.00$} & $\phantom{-}2.00${\,\small$\pm 0.00$} & $\phantom{-}2.00${\,\small$\pm 0.00$} & $\phantom{-}2.00${\,\small$\pm 0.00$} & $\phantom{-}2.00${\,\small$\pm 0.00$} & $\phantom{-}2.00${\,\small$\pm 0.00$} \\
    \bottomrule
  \end{tabular}%
  }
\end{table*}

\section{NCnet: A Toy Multi-task Model}
\subsection{Task Construction}
We construct a multi-task toy model whose variables can be directly mapped to the CHSH framework. Each input sample consists of four binary features, $x=(X_1,X_2,X_3,X_4)$, where $X_k\in{0,1}$. The first two features define two tasks on branch $A$, and the last two features define two tasks on branch $B$:
\begin{equation}
\begin{aligned}
A:
\begin{cases}
a_1=X_1 \\
a_2=X_1\oplus X_2
\end{cases}
\qquad
B:
\begin{cases}
b_1=X_3\\
b_2=X_3\oplus X_4
\end{cases}
 \end{aligned}
\end{equation}

Here, $a_1$ and $b_1$ are identity tasks, and $a_2$ and $b_2$ are XOR tasks. 
This design creates a simple but nontrivial capacity requirement: identity tasks can be represented with relatively simple features, whereas XOR tasks require nonlinear feature composition. By varying the number of hidden neurons, we can directly control whether the shared representation has sufficient capacity to support all four tasks.

The four tasks correspond to the four CHSH measurement settings. Branches $A$ and $B$ correspond to the two CHSH measurement branches, while $(a_1,a_2)$ and $(b_1,b_2)$ correspond to the two settings on each branch. For each context, NCnet is trained and evaluated on the corresponding two-task problem. The purpose is to estimate the four pairwise task correlations required by the CHSH statistic and then test whether these correlations can be jointly explained by a common shared representation.

\subsection{NCnet Architecture}
For each task-pair context $(a_i,b_j)$, NCnet is a minimal feedforward multi-task network with one shared hidden layer and two task-specific output heads. Given an input $x$, the shared hidden representation is computed as $h=\sigma(Wx+c)$, where $W$ and $c$ are shared by the two active tasks and $\sigma(\cdot)$ is the ReLU activation. Here, $h$ refers to the representation learned by the model $M_{ij}$ for the active context. The two task heads produce predictions
\begin{equation}
\begin{aligned}
\hat{a}_i=g_{a_i}(h),\quad
\hat{b}_j=g_{b_j}(h).
 \end{aligned}
\end{equation}
The two active heads are jointly trained on the shared representation within each $M_{ij}$. The four models are trained separately, so their learned representations and task heads may differ across contexts. Their pairwise correlations are combined to test compatibility with a context-independent joint representation, as described in Sec.~\ref{sec:sec2.4}.

The key capacity parameter of NCnet is the hidden width $n$, i.e., the number of neurons in the shared hidden layer. We vary $n$ to examine how representational capacity affects task compatibility. When $n$ is small, the two active tasks must compete for a limited number of hidden dimensions. When $n$ increases, the shared layer can allocate more representational degrees of freedom to different task requirements.

\subsection{Task Outcomes and CHSH Statistic}
To compute the CHSH statistic, each task prediction is converted into a binary outcome. For task $a_i$, the outcome is defined as:
\begin{equation}
\begin{aligned}
A_i(h)=
\begin{cases}
+1, & \text{if } \hat{a}_i \text{ is correct}\\
-1, & \text{otherwise}
\end{cases}
 \end{aligned}
\end{equation}
and for task $b_j$,
\begin{equation}
\begin{aligned}
B_j(h)=
\begin{cases}
+1, & \text{if } \hat{b}_j \text{ is correct} \\
-1, & \text{otherwise}
\end{cases}
 \end{aligned}
\end{equation}
For each task pair $(a_i,b_j)$, we define the task correlation as the expectation of the product of the two binary outcomes:
\begin{equation}
\begin{aligned}
E_{ij}=\mathbb{E}_{h\sim\rho_\theta(h)}[A_i(h)B_j(h)].
 \end{aligned}
\end{equation}
In experiments, this expectation is estimated by averaging over the test set:

\begin{equation}
\begin{aligned}
E_{ij}=\frac{1}{N}\sum_{m=1}^{N} A_i^{(m)}(h)B_j^{(m)}(h).
 \end{aligned}
\end{equation}

The neural CHSH statistic is then computed as:
\begin{equation}
\begin{aligned}
S_{\mathrm{NN}}=
E_{11}+E_{12}+E_{21}-E_{22}
 \end{aligned}
\end{equation}

\subsection{Training Protocol}
\label{sec:ncnet_train}
NCnet is trained under four task-pair contexts: $(a_1,b_1)$, $(a_1,b_2)$, $(a_2,b_1)$, and $(a_2,b_2)$. For each context $(a_i,b_j)$, we train a separate NCnet instance $M_{ij}$. Each $M_{ij}$ contains a shared hidden layer and two task-specific output heads, $g_{\phi_{A_i}}$ and $g_{\phi_{B_j}}$, for the active tasks $a_i$ and $b_j$. The training objective is the sum of the two active task losses. All four models use the same architecture, hidden width, input distribution, optimizer, training schedule, and evaluation protocol, and the only difference is the active task pair.

The hidden width is varied from $n=1$ to $n=20$. For each $n$, every model $M_{ij}$ is independently trained and evaluated 100 times with different random seeds. The outcomes produced by $g_{\phi_{A_i}}(h)$ and $g_{\phi_{B_j}}(h)$ are used to estimate $A_i(h)$, $B_j(h)$ and $E_{ij}$. The four correlations are then combined to compute $S_{\mathrm{NN}}$. We report the mean and standard deviation of $S_{\mathrm{NN}}$, the pairwise correlations $E_{ij}$, and the training loss over 100 runs.

This setup is intentionally minimal. By controlling the number of hidden neurons, NCnet provides a direct testbed for observing when a shared representation becomes insufficient for jointly supporting multiple tasks.

\subsection{Results and Analysis}
Figure~\ref{img:ncnet-exp} and Table~\ref{tab:hidden_dim_sweep_pairwise} reveal a capacity-dependent transition. At $n=1,2$, $S_{\mathrm{NN}}<2$ alone does not establish compatibility (Sec.~\ref{sec:sec2.4}). At $n=3$, all tasks exceed chance-level accuracy, yet $S_{\mathrm{NN}}$ peaks at $2.31\pm0.35$. The three correlations $E_{11},E_{12},E_{21}$ are $0.91$--$1.00$, whereas $E_{22}=0.53$, reflecting the difficulty of jointly fitting two XOR tasks. This exceedance rules out a context-independent joint representation of the observed correlations and indicates SIC failure.

As capacity increases, fitting becomes more balanced and exceedances diminish. At $n=4$, $E_{22}$ rises to $0.93$, training loss drops sharply, and $53\%$ of runs attain near-zero mean best-loss. The mean $S_{\mathrm{NN}}$ falls to $2.04\pm0.16$ at $n=4$ and $2.01\pm0.05$ at $n=5$, both still significantly above $2$. At $n=7$--$20$, all runs attain near-zero mean best-loss, with $E_{ij}=1$ and $S_{\mathrm{NN}}=2$. This progression identifies the Bell Boundary and supports limited shared capacity as an important source of the observed incompatibility.

The lower panel of Figure~\ref{img:ncnet-exp} provides complementary evidence of task competition. Mean gradient cosine similarity is most negative at $n=3$ (approximately $-0.812$), coinciding with the CHSH peak. With further widening, it generally becomes less negative, although not monotonically, and remains negative even where $S_{\mathrm{NN}}\approx2$. Thus, the capacity-dependent reduction in CHSH exceedances accompanies weaker average gradient conflict, without implying that conflict disappears or that SIC is established.

\section{Scaling the Bell Boundary to Large Foundation Models}
We examine whether the same capacity-dependent pattern appears in larger neural networks, including CNN, PLMs, and LLMs. 

\begin{figure}[!t]
    \centering
    \includegraphics[width=\columnwidth]{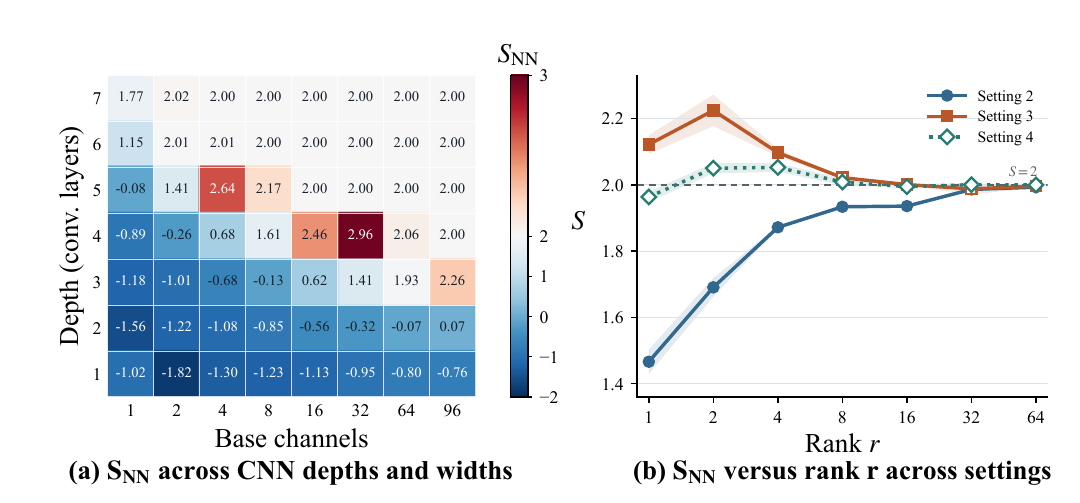}
    \caption{$S_{\mathrm{NN}}$ versus Trainable Parameter Count Across Settings. (a) Mean $S_{\mathrm{NN}}$ across five seeds for CNNs of varying depth and base channel width; values are shown in each heatmap cell. (b) Mean $S_{\mathrm{NN}}$ versus LoRA rank $r$ for Settings 2, 3, and 4. Shaded bands show one standard deviation across three seeds, the dashed line marks $S_{\mathrm{NN}}=2$.}
    \label{fig:snn_capacity_comparison}
\end{figure}
\begin{figure}[!t]
    \centering
    \includegraphics[width=\columnwidth]{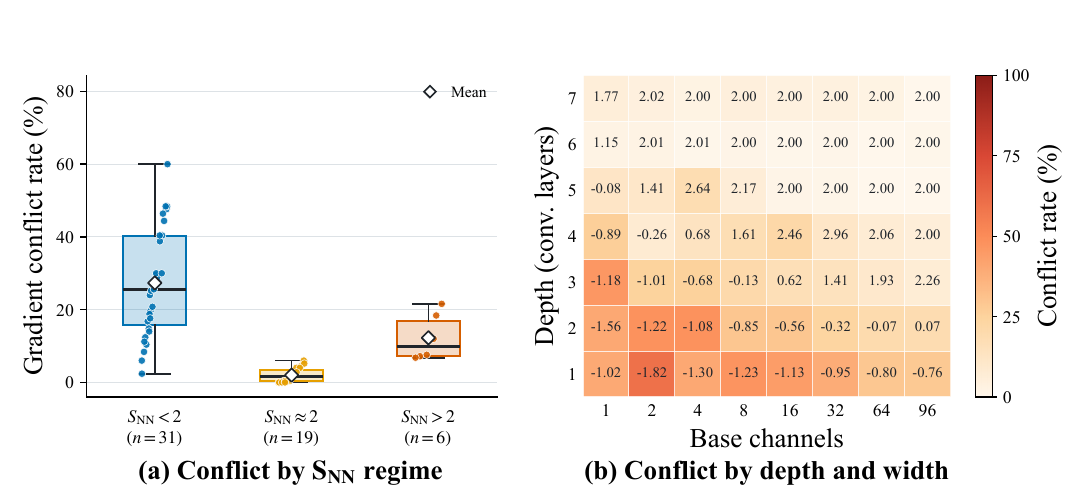}
    \caption{Gradient conflict rates for CNN task pair $A_2B_2$. (a) Rates grouped by mean $S_{\mathrm{NN}}$: $<1.97$, $1.97$--$2.03$, and $>2.03$. Points denote architectures, boxes show medians and interquartile ranges, and diamonds mark group means. (b) Rates by shared-encoder depth and base channel width; cell labels show final-epoch mean $S_{\mathrm{NN}}$.}
    \label{fig:cnn_gradient_conflict}
\end{figure}

\subsection{Experimental Setup}

\textbf{Models.}
We evaluate three classes of models with increasing scale and representational complexity. First, we use a multi-task CNN for vision experiments, consisting of a shared convolutional encoder and task-specific classification heads. 
Second, we use a BERT-based multi-task model~\cite{devlin2019bert}, where BERT serves as the shared text encoder and each task has an independent classification head. 
Third, we extend the same protocol to decoder-only LLMs (i.e., Qwen3-8B~\cite{yang2025qwen3}), where LoRA adapters~\cite{hu2022lora} are inserted into attention projection layers and trained on instruction-style classification and reasoning tasks.

\textbf{Tasks and datasets.}
For vision experiments (Setting 1), we construct a two-digit image dataset based on MNIST~\cite{lecun1998gradient}. Each image contains two digits, one on the left and one on the right. The tasks combine binary digit detection and 10-class digit recognition on the two sides.

For language experiments, Setting 2 uses multilingual paraphrase identification on PAWS-X~\cite{yang2019paws}, representing a relatively homogeneous multi-task scenario. Setting 3 combines sentiment classification~\cite{socher2013recursive}, paraphrase identification~\cite{wang-etal-2018-glue,dolan2005automatically}, commonsense reasoning~\cite{Talmor2019CommonsenseQAAQ}, and mathematical reasoning~\cite{Amini2019MathQATI}, forming a heterogeneous scenario with different objectives and reasoning requirements. Setting 4 combines mathematical reasoning on MathQA~\cite{Amini2019MathQATI} and MATH~\cite{hendrycks2021math}, reading comprehension on RACE-Middle~\cite{lai2017race}, and logical reasoning on LogiQA2.0~\cite{liu2023logiqa2} to examine shared-representation constraints in a decoder-only LLM. Detailed task assignments are summarized in Table~\ref{tab:task_combinations}.

\begin{table*}[!htb]
\centering
\setlength{\tabcolsep}{8pt}
\caption{Pairwise task correlations $E_{ij}$ and CHSH statistic
$S_{\mathrm{NN}}$ across LoRA ranks $r$ for (a) Setting 2,
(b) Setting 3, and (c) Setting 4. Values are mean$\pm$std over
three random seeds for Settings 2, 3, and 4.
\textbf{Bold} $S_{\mathrm{NN}}$ values significantly exceed $2$.}
\label{tab:language_combined}
\label{tab:exp1_homogeneous}
\label{tab:exp2_heterogeneous}
\label{tab:qwen_main_results}
\begin{tabular}{cccccc}
\toprule
\textbf{$r$} & $E_{11}$ & $E_{12}$ & $E_{21}$ & $E_{22}$ & \textbf{$S_{\mathrm{NN}}$} \\
\midrule
\multicolumn{6}{l}{\textbf{(a) Setting 2}} \\
\addlinespace[2pt]
1 & 0.743$\pm$0.021 & 0.695$\pm$0.025 & 0.689$\pm$0.027 & 0.661$\pm$0.025 & 1.466$\pm$0.037 \\
2 & 0.839$\pm$0.017 & 0.836$\pm$0.008 & 0.838$\pm$0.018 & 0.822$\pm$0.018 & 1.691$\pm$0.028 \\
4 & 0.936$\pm$0.002 & 0.915$\pm$0.005 & 0.918$\pm$0.013 & 0.896$\pm$0.016 & 1.872$\pm$0.003 \\
8 & 0.968$\pm$0.004 & 0.956$\pm$0.007 & 0.955$\pm$0.008 & 0.946$\pm$0.002 & 1.934$\pm$0.004 \\
16 & 0.969$\pm$0.004 & 0.961$\pm$0.003 & 0.957$\pm$0.001 & 0.951$\pm$0.002 & 1.936$\pm$0.003 \\
32 & 0.992$\pm$0.006 & 0.992$\pm$0.007 & 0.989$\pm$0.003 & 0.988$\pm$0.005 & 1.986$\pm$0.003 \\
64 & 0.996$\pm$0.002 & 0.996$\pm$0.003 & 0.995$\pm$0.001 & 0.994$\pm$0.002 & 1.993$\pm$0.001 \\
\midrule
\multicolumn{6}{l}{\textbf{(b) Setting 3}} \\
\addlinespace[2pt]
1 & 0.993$\pm$0.001 & 0.262$\pm$0.049 & 0.792$\pm$0.012 & -0.075$\pm$0.013 & $\mathbf{2.122}\pm$0.028 \\
2 & 0.998$\pm$0.000 & 0.653$\pm$0.048 & 0.945$\pm$0.009 & 0.372$\pm$0.035 & $\mathbf{2.224}\pm$0.047 \\
4 & 0.999$\pm$0.001 & 0.908$\pm$0.008 & 0.984$\pm$0.003 & 0.794$\pm$0.017 & $\mathbf{2.097}\pm$0.011 \\
8 & 1.000$\pm$0.001 & 0.973$\pm$0.002 & 0.993$\pm$0.002 & 0.944$\pm$0.005 & $\mathbf{2.021}\pm$0.004 \\
16 & 1.000$\pm$0.000 & 0.989$\pm$0.003 & 0.993$\pm$0.003 & 0.981$\pm$0.004 & 2.001$\pm$0.002 \\
32 & 1.000$\pm$0.000 & 0.978$\pm$0.015 & 0.996$\pm$0.000 & 0.986$\pm$0.007 & 1.988$\pm$0.018 \\
64 & 1.000$\pm$0.000 & 0.992$\pm$0.001 & 0.994$\pm$0.002 & 0.990$\pm$0.004 & 1.996$\pm$0.002 \\
\midrule
\multicolumn{6}{l}{\textbf{(c) Setting 4}} \\
\addlinespace[2pt]
1 & 0.943$\pm$0.004 & 0.902$\pm$0.005 & 0.352$\pm$0.010 & 0.234$\pm$0.006 & 1.963$\pm$0.015 \\
2 & 0.989$\pm$0.003 & 0.976$\pm$0.003 & 0.664$\pm$0.008 & 0.579$\pm$0.007 & $\mathbf{2.050}\pm$0.016 \\
4 & 0.991$\pm$0.003 & 0.988$\pm$0.003 & 0.830$\pm$0.007 & 0.756$\pm$0.006 & $\mathbf{2.052}\pm$0.013 \\
8 & 0.998$\pm$0.002 & 0.995$\pm$0.002 & 0.928$\pm$0.005 & 0.912$\pm$0.005 & 2.008$\pm$0.010 \\
16 & 1.000$\pm$0.001 & 0.997$\pm$0.002 & 0.981$\pm$0.004 & 0.983$\pm$0.002 & 1.994$\pm$0.004 \\
32 & 1.000$\pm$0.001 & 1.000$\pm$0.001 & 0.999$\pm$0.001 & 0.998$\pm$0.001 & 2.000$\pm$0.002 \\
64 & 1.000$\pm$0.000 & 0.999$\pm$0.000 & 1.000$\pm$0.000 & 0.999$\pm$0.001 & 2.000$\pm$0.001 \\
\bottomrule
\end{tabular}
\end{table*}

\textbf{Capacity control.}
In the vision experiments, we vary the capacity of the shared convolutional encoder by adjusting its depth through the number of convolutional layers and its width through the number of channels. Specifically, we vary the channel width over $\{1,2,4,8,16,32,64,96\}$ and the depth from $1$ to $7$, thereby changing the number of trainable parameters available for shared feature extraction. In both the BERT and Qwen3-8B experiments, we vary the LoRA rank $r$ over $\{1,2,4,8,16,32,64\}$, which controls the effective number of trainable adaptation parameters. 

\textbf{Training and evaluation.}
All models are trained in a multi-task setting using the sum of task losses. All tasks are formulated as classification tasks and trained using cross-entropy loss. CNN models are trained for 50 epochs, with each experiment repeated 5 times using different random seeds. BERT models are trained for 100 epochs in Setting 2 and 200 epochs in Setting 3, while Qwen3-8B is trained for 10 epochs in Setting 4. Each BERT and Qwen3-8B experiment is repeated 3 times using different random seeds. 
For Settings 2, 3, and 4, examples are randomly paired across tasks. We independently shuffle the training datasets for $a_1$, $a_2$, $b_1$, and $b_2$, then form each task pair $(a_i,b_j)$ by drawing examples sequentially from the corresponding shuffled datasets. When a dataset has insufficient examples, repeated sampling supplies the additional examples needed for pairing. The product of the two binary correctness outcomes is computed for each sampled pair and averaged to obtain $E_{ij}$. We report the mean and standard deviation of task accuracy $Acc.$, pairwise task correlations $E_{ij} $, and the CHSH statistic $S_{\mathrm{NN}}$. Full training details are provided in Appendix~\ref{app:training_details}. 

\subsection{Experimental Results}

\textbf{Vision (CNN).}
Table~\ref{tab:cnn_snn_train} shows that $S_{\mathrm{NN}}$ remains below 2 at low CNN capacities, significantly exceeds 2 at intermediate capacities, and returns to approximately 2 as capacity further increases. The intermediate-capacity exceedance thus delineates a clear Bell Boundary region. This Bell Boundary region shifts toward smaller channel widths as encoder depth increases. Figure~\ref{fig:cnn_gradient_conflict} shows that gradient conflict in $A_2B_2$ is lowest for architectures near $S_{\mathrm{NN}}=2$, while both the below-bound and above-bound groups exhibit stronger conflict. Together, these results associate the capacity-dependent boundary with persistent task competition; $S_{\mathrm{NN}}<2$ alone does not establish compatibility, while exceedances indicate incompatibility with a context-independent joint probability model of task outcomes.

\textbf{Language (PLM/LLM).}
Table~\ref{tab:language_combined} and Figure~\ref{fig:snn_capacity_comparison}(b) show contrasting capacity-dependent patterns. BERT's homogeneous tasks (Setting 2) approach $S_{\mathrm{NN}}=2$ from below. Their shared paraphrase objective offers a plausible explanation for the four correlations improving together and remaining similar: when $E_{ij}\approx c$, $S_{\mathrm{NN}}\approx2c\leq2$. Homogeneity alone, however, does not guarantee this bound. BERT's heterogeneous tasks (Setting 3) exceed the bound at limited adaptation capacity. Qwen3-8B (Setting 4) exhibits a smaller exceedance at intermediate LoRA ranks, extending the Bell Boundary observation to a large decoder-only model. In both heterogeneous settings, exceedances diminish as adaptation capacity increases and pairwise correlations become more balanced.

Figure~\ref{fig:bert_qwen_gradient_comparison} shows that gradient conflict generally decreases with increasing LoRA rank in Qwen but varies across task pairs in BERT. BERT's $A_2B_2$ conflict rate tends to increase with $r$, which can be interpreted in light of SIC. Under the task-specific decomposition assumed by SIC, the two tasks occupy independent, orthogonal effective parameter subspaces, so their nonzero gradients have zero inner product and hence zero cosine similarity. The narrowing central spread and the center shifting toward zero in Figure~\ref{fig:bert_qwen_gradient_comparison}(b) are consistent with increasing capacity allowing the tasks to approach this orthogonal configuration. Exact orthogonality yields no gradient conflict; during finite training, however, small deviations from orthogonality can produce slightly negative cosine values. As the distribution shifts toward zero from the positive side, such values can become more frequent even as their typical magnitudes decrease. Since the conflict rate records only whether the gradient inner product is negative, its increase does not by itself establish stronger task competition.

Together, the CNN, BERT, and Qwen results suggest that the Bell Boundary depends on both available capacity and task composition, and extends across vision and language models. Increasing shared capacity reduces CHSH exceedances and can alleviate gradient conflict. These findings suggest that limited shared representational capacity may sustain gradient competition between tasks, which could contribute to the observed violations of the neural CHSH bound ($|S_{\mathrm{NN}}|>2$). The recurrence of this capacity-dependent pattern across distinct architectures, modalities, and task settings provides evidence that the Bell Boundary reflects a general constraint on shared representations in multi-task learning.

\begin{figure*}[!tp]
    \centering
    \includegraphics[width=0.8\textwidth]{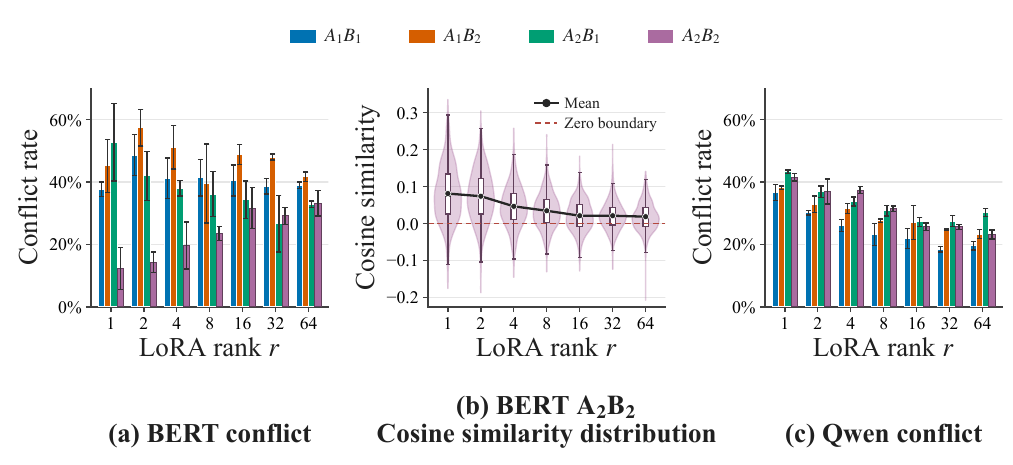}
    \caption{Gradient conflict across LoRA ranks $r$. (a, c) Training-averaged conflict rates for the four task pairs in BERT and Qwen, respectively; error bars show standard deviations across seeds. (b) Gradient cosine similarity distributions for BERT's $A_2B_2$ pair: violins show density, boxes show medians and interquartile ranges, and black circles mark means. The red dashed line marks the zero-cosine conflict boundary.}
    \label{fig:bert_qwen_gradient_comparison}
\end{figure*}

\section{Conclusion}
This paper introduces the Bell Boundary as a measurable capacity boundary of shared representations. We show that, under the adopted CHSH construction, task correlations jointly supported by a common shared representation should satisfy $|S_{\mathrm{NN}}|\leq2$. Experiments on NCnet, CNNs, PLMs, and LLMs show that $S_{\mathrm{NN}}$ can identify the capacity range where a shared representation is insufficient to support multiple task contexts. Homogeneous tasks are easier to support, while heterogeneous tasks more clearly expose the boundary before sufficient capacity is reached.

This study has several limitations. First, the CHSH statistic is meaningful only after the model exceeds chance-level prediction. Second, the construction requires task outcomes to be converted into binary correctness variables, which may lose information in tasks with richer output structures. Third, the boundary depends on task grouping, capacity control, and evaluation protocol, so its location should not be viewed as a fixed property of a model. Future work can extend this framework to more complex task structures, continuous outputs, adaptive task grouping, and practical training strategies that use the Bell Boundary to guide capacity expansion or modularization.

\appendix

\section{Significance Testing of the Bell-Type Statistic}
\label{app:sigtest}

For each hidden width $d$, we use a right-tailed one-sample $t$-test on
$S_{\mathrm{NN}}$ values from 100 independent random seeds to assess whether
their mean exceeds $2$.
Results with unadjusted $p<0.05$ are bolded in
Table~\ref{tab:hidden_dim_sweep_pairwise}.

For Settings 1--4, we apply the same right-tailed one-sample $t$-test against the bound $2$ at each capacity configuration, using five random seeds for Setting 1 and three for Settings 2--4. We use unadjusted $p<0.05$ without correction for multiple comparisons. Setting 1 shows significant exceedances at intermediate CNN capacities, whereas Setting 2 shows none at any tested LoRA rank. Setting 3 significantly exceeds $2$ at $r\in\{1,2,4,8\}$, and Setting 4 at $r\in\{2,4\}$; higher ranks show no significant exceedance.

\section{Detailed Synthetic Experimental Setup}
\label{app:synth_setup}

This appendix provides the full training details of the synthetic NCnet experiment summarized in Section~4. All results are produced by the released script \texttt{bell\_nn\_experiment.py} (driven by \texttt{run\_all\_n.py}), which we summarize below.

\textbf{Data.} The synthetic dataset is the complete set of $2^4=16$ binary inputs $X\in\{0,1\}^4$. Each input appears twice in the training set (32 training samples) and once in the evaluation set (16 samples); since the input space is fully enumerated, the training and evaluation sets coincide and we report training-set statistics. The Alice/Bob labels are $a_1=X_1$, $a_2=X_1\oplus X_2$, $b_1=X_3$, $b_2=X_3\oplus X_4$.

\textbf{Architecture.} Each two-task model is a one-hidden-layer feedforward network $X\in\mathbb{R}^4 \xrightarrow{\;fc_1\;} h\in\mathbb{R}^n \xrightarrow{\;fc_2\;} y\in\mathbb{R}^2$ with ReLU activations applied after both $fc_1$ and $fc_2$, where $n\in\{1,\ldots,20\}$ is the shared hidden width and the two output scalars are the Alice and Bob task heads. The hidden layer $fc_1$ is the only shared representation; the two heads share no parameters beyond it.

\textbf{Training.} Each model is trained for up to 20000 epochs with binary cross-entropy (averaged over both heads and all samples) under the Adam optimizer (learning rate $1\times10^{-2}$, default $\beta$ and weight decay), with early stopping once the training loss falls below $10^{-6}$. Within each run, the checkpoint of lowest training loss is retained for evaluation.

\textbf{Aggregation.} For each task pair $(a_i,b_j)$ and hidden width $n$, we train 100 independent models with random seeds $0,\ldots,99$, each seed fixing both the \texttt{torch} and \texttt{NumPy} random number generators. The consistency indicator $A_i$ ($B_j$) takes the value $+1$ when the corresponding head predicts the binary label correctly and $-1$ otherwise. For each seed $s$, the pairwise correlation $E_{ij}^{(s)}=\langle A_i B_j\rangle$ is the arithmetic mean over the 16 inputs, and the per-seed CHSH statistic is $S_{\mathrm{NN}}^{(s)}=E_{11}^{(s)}+E_{12}^{(s)}+E_{21}^{(s)}-E_{22}^{(s)}$. We report the mean and standard deviation over the 100 seeds. Figure~\ref{img:ncnet-exp} additionally reports the average training loss and the gradient cosine similarity defined below.

\textbf{Gradient competition metric.} To quantify gradient competition, we backpropagate each head's loss separately through the shared hidden layer ($fc_1$) and record the cosine similarity between the two resulting gradient vectors; values near $-1$ indicate strong conflict. This metric is sampled every 100 epochs. The lower panel of Figure~\ref{img:ncnet-exp} reports the mean over the four task-pair contexts and all training phases.

\section{Training Details for Large Foundation Models}
\label{app:training_details}

Table~\ref{tab:real_world_training_details} summarizes the optimization and adaptation hyperparameters for Settings 1--4. Each task-pair context is trained separately. We report the settings used by the corresponding training pipelines, including differences between the original PAWS-X pipeline and the later BERT and Qwen pipelines.

\begin{table*}[t]
\centering
\small
\setlength{\tabcolsep}{4pt}
\renewcommand{\arraystretch}{1.15}
\caption{Training hyperparameters for Settings 1--4. CNN uses its full training set; language sample caps are per task. A dash indicates an inapplicable parameter or one not recorded in the available training logs. The Qwen learning-rate floors are specified below.}
\label{tab:real_world_training_details}
\begin{tabularx}{\textwidth}{>{\raggedright\arraybackslash}p{0.27\textwidth}*{4}{>{\centering\arraybackslash}X}}
\toprule
\textbf{Parameter} & \makecell{\textbf{CNN}\\\textbf{Setting 1}} & \makecell{\textbf{BERT}\\\textbf{Setting 2}} & \makecell{\textbf{BERT}\\\textbf{Setting 3}} & \makecell{\textbf{Qwen3-8B}\\\textbf{Setting 4}} \\
\midrule
Optimizer & AdamW & --- & AdamW & AdamW \\
Initial learning rate & $10^{-3}$ & $3\times10^{-4}$ & $10^{-3}$ (LoRA) & $2\times10^{-4}$ \\
Classification-head learning rate & $10^{-3}$ & --- & $2.5\times10^{-3}$ & --- \\
Learning-rate decay & Cosine & Plateau & Cosine & Cosine \\
Linear warmup & 5 epochs & $5\%$ & $3\%$ & $5\%$ \\
Learning-rate floor & $0$ & $10^{-6}$ & $10^{-6}$ & \makecell{$10^{-6}$ or\\$10^{-5}$} \\
Weight decay & $10^{-4}$ & --- & $0.01$ & $0.01$ \\
Gradient clipping norm & None & --- & $1.0$ & $1.0$ \\
Batch size & $512$ & $64$ & $64$ pairs & $8$ examples \\
Gradient accumulation steps & $1$ & --- & $1$ & $3$ \\
Training epochs & $50$ & At most $100$ & $200$ & $10$ \\
Maximum sequence length & --- & --- & $192$ & $512$ \\
Training samples / cap & $4000$ & $5000$ & $3000$ & $3000$ \\
Evaluation sample cap & --- & \makecell{$625$ dev /\\$625$ test} & $1500$ & $500$ \\
LoRA scaling $\alpha$ & --- & $2r$ & $2r$ & $2r$ \\
LoRA dropout & --- & $0.05$ & $0$ & $0$ \\
LoRA rank $r$ & --- & \makecell{$1,2,4,8,$\\$16,32,64$} & \makecell{$1,2,4,8,$\\$16,32,64$} & \makecell{$1,2,4,8,$\\$16,32,64$} \\
Reported seed count & $5$ & $3$ & $3$ & $3$ \\
\bottomrule
\end{tabularx}
\end{table*}

\textbf{Optimization.} CNN, BERT Setting 3, and Qwen use PyTorch AdamW with its default moment coefficients $(\beta_1,\beta_2)=(0.9,0.999)$ and numerical-stability constant $\epsilon=10^{-8}$. Their learning rates follow a half-cycle cosine decay after linear warmup; the schedule advances after each optimizer update. The CNN training set contains 4000 images covering 100 left--right digit combinations, with 40 images per combination. CNN uses encoder dropout $0.1$ and linear task heads, with the depth and channel-width sweep specified in Section~5. BERT Setting 3 and Qwen run for the fixed epoch budgets in the table without early stopping.

\textbf{PAWS-X (Setting 2).} The original multilingual BERT pipeline uses hard parameter sharing and equally weighted task losses. LoRA targets the attention query and value projections, with $\alpha/r=2$ and dropout $0.05$. After $5\%$ warmup, the learning-rate scheduler reduces the rate by a factor of $0.5$ on a plateau, with patience $10$, relative threshold $10^{-3}$, and cooldown $3$ epochs. The early-stopping configuration has patience $10$, minimum improvement $0$, and equal task weights $(0.5,0.5)$. The training budget is at most 100 epochs, with the separate development and test caps listed in the table.

\textbf{Mixed-task BERT (Setting 3).} We use BERT-base-uncased with frozen pretrained encoder weights and train LoRA modules on the query and value projections in every attention layer, together with task-specific heads. Each head has a 128-dimensional hidden layer, ReLU activation, and zero head dropout. Classification and multiple-choice tasks use cross-entropy, and the two active task losses are summed with unit weights. A batch contains 64 paired examples, hence 64 examples from each active task. Training uses bfloat16 autocast.

\textbf{Qwen3-8B (Setting 4).} We freeze the pretrained transformer and language-model output head and train only LoRA modules on the attention query and value projections (\texttt{q\_proj} and \texttt{v\_proj}), using a bfloat16 backbone, float32 adapters, and gradient checkpointing. Each microbatch contains four examples from each active task; three microbatches are accumulated, giving an effective batch size of 24 and 250 optimizer updates per epoch (2500 in total). Supervised fine-tuning minimizes next-token cross-entropy on the answer and end-of-turn tokens, with prompt tokens masked from the loss. At the $1{:}1$ task mixture, the implemented objective is the sum of the two task-wise mean answer-token losses. The 512-token limit covers the prompt and answer together; long passages are truncated while retaining the question, options, and answer. Available saved runs use a learning-rate floor of $10^{-5}$ for $r=2$, for $A_2B_1$ at $r=4$, and for $A_2B_2$ at $r\in\{4,8,16,32\}$; the other available configurations use $10^{-6}$.

\textbf{Measurement protocol.} CNN and BERT Setting 3 statistics use the final-epoch training outcomes. Qwen correlations use a fixed probe drawn from the training pool, with 3000 examples per task, probe seed 1234, and paired probe batches of size 8, evaluated at each epoch end. Correctness on this probe is computed by parsing teacher-forced answer-position argmax predictions, and the task outcomes are paired by probe order. Separate held-out Qwen evaluation uses greedy decoding with at most 128 new tokens and exact matching of the parsed answer. Gradient competition is measured on the trainable shared parameters once per epoch for CNN, BERT Setting 3, and Qwen; Qwen averages the metrics over paired probe batches. Seed counts in the table refer to the reported result summaries.

\section{Detailed CNN Training Accuracies}
\label{app:cnn_train_acc}

Tables~\ref{tab:cnn_train_acc_d1}--\ref{tab:cnn_train_acc_d7} report final-epoch training accuracies for each task in the four CNN task-pair contexts across channel widths, with one table per shared-encoder depth. Within each context, the $A_i$ and $B_j$ columns correspond to \texttt{train\_acc\_a} and \texttt{train\_acc\_b}, respectively. Values are percentages, reported as mean $\pm$ sample standard deviation over five random seeds ($42$--$46$; standard deviation computed with denominator $5-1$).

\begin{table*}[t]
  \centering
  \small
  \caption{Task accuracy (\%) for CNN in Setting 2 at shared-encoder depth $1$. Each entry is mean $\pm$ sample standard deviation over five random seeds.}
  \label{tab:cnn_train_acc_d1}
  \setlength{\tabcolsep}{3pt}
  \renewcommand{\arraystretch}{1.15}
  \resizebox{\textwidth}{!}{%
  \begin{tabular}{ccccccccc}
    \toprule
    \multirow{2}{*}{Channel width} & \multicolumn{2}{c}{$A_1B_1$} & \multicolumn{2}{c}{$A_1B_2$} & \multicolumn{2}{c}{$A_2B_1$} & \multicolumn{2}{c}{$A_2B_2$} \\
    \cmidrule(lr){2-3}\cmidrule(lr){4-5}\cmidrule(lr){6-7}\cmidrule(lr){8-9}
      & $A_1$ & $B_1$ & $A_1$ & $B_2$ & $A_2$ & $B_1$ & $A_2$ & $B_2$ \\
    \midrule
    1 & $90.00 \pm 0.00$ & $90.00 \pm 0.00$ & $74.00 \pm 35.78$ & $10.02 \pm 0.03$ & $9.95 \pm 0.13$ & $90.00 \pm 0.00$ & $10.02 \pm 0.03$ & $10.07 \pm 0.16$ \\
    2 & $74.00 \pm 35.78$ & $72.48 \pm 35.08$ & $88.35 \pm 3.69$ & $9.52 \pm 1.05$ & $9.30 \pm 0.88$ & $90.00 \pm 0.00$ & $10.07 \pm 0.45$ & $9.62 \pm 0.71$ \\
    4 & $85.69 \pm 4.95$ & $90.00 \pm 0.00$ & $89.99 \pm 0.02$ & $10.40 \pm 0.66$ & $10.57 \pm 1.38$ & $89.27 \pm 1.63$ & $10.25 \pm 0.76$ & $10.06 \pm 0.33$ \\
    8 & $90.00 \pm 0.00$ & $90.00 \pm 0.00$ & $89.98 \pm 0.04$ & $10.70 \pm 1.13$ & $10.58 \pm 1.38$ & $90.00 \pm 0.00$ & $10.93 \pm 0.89$ & $10.49 \pm 1.18$ \\
    16 & $90.00 \pm 0.00$ & $90.00 \pm 0.00$ & $90.00 \pm 0.00$ & $12.80 \pm 1.15$ & $11.73 \pm 0.74$ & $90.00 \pm 0.00$ & $11.97 \pm 0.78$ & $13.28 \pm 1.10$ \\
    32 & $90.00 \pm 0.00$ & $90.00 \pm 0.00$ & $90.00 \pm 0.00$ & $15.39 \pm 1.24$ & $14.71 \pm 1.02$ & $90.00 \pm 0.00$ & $14.90 \pm 1.61$ & $15.32 \pm 0.59$ \\
    64 & $90.00 \pm 0.00$ & $90.00 \pm 0.00$ & $90.00 \pm 0.00$ & $16.93 \pm 0.29$ & $16.74 \pm 0.40$ & $90.00 \pm 0.00$ & $17.14 \pm 0.83$ & $18.01 \pm 0.47$ \\
    96 & $90.00 \pm 0.00$ & $90.00 \pm 0.00$ & $90.00 \pm 0.00$ & $18.13 \pm 0.24$ & $16.97 \pm 0.20$ & $90.00 \pm 0.00$ & $17.63 \pm 0.50$ & $18.67 \pm 0.46$ \\
    \bottomrule
  \end{tabular}%
  }
\end{table*}

\begin{table*}[t]
  \centering
  \small
  \caption{Task accuracy (\%) for CNN in Setting 2 at shared-encoder depth $2$. Each entry is mean $\pm$ sample standard deviation over five random seeds.}
  \label{tab:cnn_train_acc_d2}
  \setlength{\tabcolsep}{3pt}
  \renewcommand{\arraystretch}{1.15}
  \resizebox{\textwidth}{!}{%
  \begin{tabular}{ccccccccc}
    \toprule
    \multirow{2}{*}{Channel width} & \multicolumn{2}{c}{$A_1B_1$} & \multicolumn{2}{c}{$A_1B_2$} & \multicolumn{2}{c}{$A_2B_1$} & \multicolumn{2}{c}{$A_2B_2$} \\
    \cmidrule(lr){2-3}\cmidrule(lr){4-5}\cmidrule(lr){6-7}\cmidrule(lr){8-9}
      & $A_1$ & $B_1$ & $A_1$ & $B_2$ & $A_2$ & $B_1$ & $A_2$ & $B_2$ \\
    \midrule
    1 & $72.50 \pm 35.10$ & $72.36 \pm 35.04$ & $86.07 \pm 5.18$ & $10.09 \pm 0.13$ & $10.08 \pm 0.12$ & $78.11 \pm 26.59$ & $10.13 \pm 0.22$ & $9.61 \pm 0.67$ \\
    2 & $90.00 \pm 0.00$ & $90.01 \pm 0.01$ & $89.84 \pm 0.37$ & $10.13 \pm 0.16$ & $11.14 \pm 1.91$ & $89.89 \pm 0.25$ & $11.51 \pm 1.92$ & $10.78 \pm 1.12$ \\
    4 & $90.00 \pm 0.00$ & $90.00 \pm 0.00$ & $90.00 \pm 0.00$ & $11.72 \pm 0.95$ & $11.96 \pm 1.15$ & $90.00 \pm 0.00$ & $15.58 \pm 1.28$ & $13.91 \pm 1.30$ \\
    8 & $90.00 \pm 0.00$ & $90.00 \pm 0.00$ & $90.00 \pm 0.00$ & $16.12 \pm 1.69$ & $15.04 \pm 1.73$ & $90.00 \pm 0.00$ & $18.20 \pm 0.97$ & $18.60 \pm 1.83$ \\
    16 & $90.00 \pm 0.00$ & $90.00 \pm 0.00$ & $90.00 \pm 0.00$ & $21.29 \pm 0.76$ & $19.79 \pm 0.73$ & $90.00 \pm 0.00$ & $21.86 \pm 0.21$ & $22.71 \pm 0.80$ \\
    32 & $90.00 \pm 0.00$ & $90.00 \pm 0.00$ & $90.00 \pm 0.00$ & $24.59 \pm 1.13$ & $23.00 \pm 0.78$ & $90.00 \pm 0.00$ & $25.25 \pm 0.46$ & $26.79 \pm 0.61$ \\
    64 & $90.00 \pm 0.00$ & $90.00 \pm 0.00$ & $90.00 \pm 0.00$ & $28.88 \pm 0.68$ & $27.24 \pm 0.93$ & $90.00 \pm 0.00$ & $29.01 \pm 0.98$ & $30.49 \pm 1.00$ \\
    96 & $90.00 \pm 0.00$ & $90.04 \pm 0.05$ & $90.00 \pm 0.00$ & $31.49 \pm 0.74$ & $29.97 \pm 0.65$ & $90.00 \pm 0.00$ & $30.92 \pm 0.89$ & $32.60 \pm 0.79$ \\
    \bottomrule
  \end{tabular}%
  }
\end{table*}

\begin{table*}[t]
  \centering
  \small
  \caption{Task accuracy (\%) for CNN in Setting 2 at shared-encoder depth $3$. Each entry is mean $\pm$ sample standard deviation over five random seeds.}
  \label{tab:cnn_train_acc_d3}
  \setlength{\tabcolsep}{3pt}
  \renewcommand{\arraystretch}{1.15}
  \resizebox{\textwidth}{!}{%
  \begin{tabular}{ccccccccc}
    \toprule
    \multirow{2}{*}{Channel width} & \multicolumn{2}{c}{$A_1B_1$} & \multicolumn{2}{c}{$A_1B_2$} & \multicolumn{2}{c}{$A_2B_1$} & \multicolumn{2}{c}{$A_2B_2$} \\
    \cmidrule(lr){2-3}\cmidrule(lr){4-5}\cmidrule(lr){6-7}\cmidrule(lr){8-9}
      & $A_1$ & $B_1$ & $A_1$ & $B_2$ & $A_2$ & $B_1$ & $A_2$ & $B_2$ \\
    \midrule
    1 & $90.00 \pm 0.00$ & $89.99 \pm 0.02$ & $89.83 \pm 0.38$ & $10.79 \pm 1.01$ & $11.69 \pm 2.35$ & $90.00 \pm 0.00$ & $12.58 \pm 0.76$ & $11.19 \pm 1.65$ \\
    2 & $90.00 \pm 0.00$ & $90.00 \pm 0.00$ & $89.97 \pm 0.08$ & $13.74 \pm 2.20$ & $14.16 \pm 1.90$ & $89.99 \pm 0.03$ & $15.01 \pm 1.37$ & $15.09 \pm 1.82$ \\
    4 & $90.00 \pm 0.00$ & $90.00 \pm 0.00$ & $90.00 \pm 0.00$ & $18.30 \pm 2.47$ & $18.75 \pm 1.74$ & $90.00 \pm 0.00$ & $20.82 \pm 1.58$ & $21.16 \pm 1.86$ \\
    8 & $90.00 \pm 0.00$ & $90.00 \pm 0.00$ & $90.00 \pm 0.00$ & $27.00 \pm 2.37$ & $26.42 \pm 1.72$ & $90.00 \pm 0.00$ & $30.51 \pm 1.87$ & $31.81 \pm 2.55$ \\
    16 & $90.01 \pm 0.02$ & $90.00 \pm 0.00$ & $90.00 \pm 0.00$ & $41.44 \pm 3.01$ & $36.53 \pm 1.58$ & $90.00 \pm 0.00$ & $39.90 \pm 1.00$ & $42.64 \pm 0.74$ \\
    32 & $91.51 \pm 0.53$ & $90.88 \pm 0.15$ & $90.92 \pm 0.11$ & $52.97 \pm 1.17$ & $49.77 \pm 1.40$ & $90.90 \pm 0.04$ & $49.04 \pm 1.09$ & $51.01 \pm 1.09$ \\
    64 & $94.14 \pm 0.76$ & $92.70 \pm 0.36$ & $91.00 \pm 0.15$ & $64.26 \pm 1.61$ & $63.29 \pm 2.29$ & $91.01 \pm 0.06$ & $53.07 \pm 0.77$ & $55.66 \pm 0.75$ \\
    96 & $97.01 \pm 0.92$ & $95.06 \pm 0.85$ & $90.81 \pm 0.21$ & $73.45 \pm 1.96$ & $71.33 \pm 1.78$ & $90.83 \pm 0.16$ & $54.53 \pm 0.63$ & $56.86 \pm 0.62$ \\
    \bottomrule
  \end{tabular}%
  }
\end{table*}

\begin{table*}[t]
  \centering
  \small
  \caption{Task accuracy (\%) for CNN in Setting 2 at shared-encoder depth $4$. Each entry is mean $\pm$ sample standard deviation over five random seeds.}
  \label{tab:cnn_train_acc_d4}
  \setlength{\tabcolsep}{3pt}
  \renewcommand{\arraystretch}{1.15}
  \resizebox{\textwidth}{!}{%
  \begin{tabular}{ccccccccc}
    \toprule
    \multirow{2}{*}{Channel width} & \multicolumn{2}{c}{$A_1B_1$} & \multicolumn{2}{c}{$A_1B_2$} & \multicolumn{2}{c}{$A_2B_1$} & \multicolumn{2}{c}{$A_2B_2$} \\
    \cmidrule(lr){2-3}\cmidrule(lr){4-5}\cmidrule(lr){6-7}\cmidrule(lr){8-9}
      & $A_1$ & $B_1$ & $A_1$ & $B_2$ & $A_2$ & $B_1$ & $A_2$ & $B_2$ \\
    \midrule
    1 & $90.00 \pm 0.00$ & $89.99 \pm 0.03$ & $89.99 \pm 0.02$ & $15.66 \pm 2.33$ & $15.94 \pm 2.15$ & $90.00 \pm 0.00$ & $16.49 \pm 1.53$ & $17.09 \pm 1.76$ \\
    2 & $90.00 \pm 0.00$ & $90.00 \pm 0.00$ & $90.00 \pm 0.00$ & $24.17 \pm 3.26$ & $21.59 \pm 2.20$ & $90.00 \pm 0.00$ & $28.97 \pm 2.86$ & $30.03 \pm 3.03$ \\
    4 & $90.06 \pm 0.06$ & $90.46 \pm 0.48$ & $90.14 \pm 0.27$ & $40.38 \pm 2.72$ & $39.33 \pm 2.27$ & $90.44 \pm 0.34$ & $41.51 \pm 2.27$ & $44.13 \pm 3.24$ \\
    8 & $92.26 \pm 0.63$ & $91.53 \pm 0.16$ & $90.95 \pm 0.14$ & $57.24 \pm 1.92$ & $55.01 \pm 1.48$ & $90.96 \pm 0.08$ & $51.32 \pm 0.95$ & $54.04 \pm 0.29$ \\
    16 & $98.34 \pm 0.51$ & $96.32 \pm 0.67$ & $90.32 \pm 0.18$ & $78.80 \pm 2.13$ & $75.95 \pm 2.04$ & $90.50 \pm 0.10$ & $54.10 \pm 0.26$ & $57.36 \pm 0.26$ \\
    32 & $100.00 \pm 0.00$ & $99.98 \pm 0.02$ & $92.31 \pm 1.31$ & $99.28 \pm 0.34$ & $98.79 \pm 0.65$ & $91.08 \pm 0.67$ & $61.91 \pm 0.60$ & $64.54 \pm 1.00$ \\
    64 & $100.00 \pm 0.00$ & $100.00 \pm 0.00$ & $99.99 \pm 0.02$ & $100.00 \pm 0.00$ & $100.00 \pm 0.00$ & $99.76 \pm 0.39$ & $98.33 \pm 0.55$ & $98.11 \pm 0.47$ \\
    96 & $100.00 \pm 0.00$ & $100.00 \pm 0.00$ & $100.00 \pm 0.00$ & $100.00 \pm 0.00$ & $100.00 \pm 0.00$ & $100.00 \pm 0.00$ & $100.00 \pm 0.00$ & $100.00 \pm 0.00$ \\
    \bottomrule
  \end{tabular}%
  }
\end{table*}

\begin{table*}[t]
  \centering
  \small
  \caption{Task accuracy (\%) for CNN in Setting 2 at shared-encoder depth $5$. Each entry is mean $\pm$ sample standard deviation over five random seeds.}
  \label{tab:cnn_train_acc_d5}
  \setlength{\tabcolsep}{3pt}
  \renewcommand{\arraystretch}{1.15}
  \resizebox{\textwidth}{!}{%
  \begin{tabular}{ccccccccc}
    \toprule
    \multirow{2}{*}{Channel width} & \multicolumn{2}{c}{$A_1B_1$} & \multicolumn{2}{c}{$A_1B_2$} & \multicolumn{2}{c}{$A_2B_1$} & \multicolumn{2}{c}{$A_2B_2$} \\
    \cmidrule(lr){2-3}\cmidrule(lr){4-5}\cmidrule(lr){6-7}\cmidrule(lr){8-9}
      & $A_1$ & $B_1$ & $A_1$ & $B_2$ & $A_2$ & $B_1$ & $A_2$ & $B_2$ \\
    \midrule
    1 & $90.00 \pm 0.00$ & $90.33 \pm 0.45$ & $90.01 \pm 0.02$ & $30.58 \pm 1.77$ & $28.45 \pm 2.12$ & $90.13 \pm 0.25$ & $27.40 \pm 3.25$ & $28.66 \pm 2.70$ \\
    2 & $92.48 \pm 0.80$ & $93.91 \pm 1.05$ & $90.37 \pm 0.33$ & $58.91 \pm 2.09$ & $54.17 \pm 3.05$ & $90.56 \pm 0.30$ & $45.57 \pm 1.39$ & $48.83 \pm 1.14$ \\
    4 & $98.00 \pm 0.57$ & $98.71 \pm 0.76$ & $90.81 \pm 0.24$ & $88.21 \pm 0.80$ & $87.47 \pm 1.42$ & $91.67 \pm 0.66$ & $59.50 \pm 1.97$ & $63.40 \pm 1.48$ \\
    8 & $100.00 \pm 0.01$ & $99.98 \pm 0.02$ & $98.28 \pm 0.37$ & $99.75 \pm 0.10$ & $99.45 \pm 0.05$ & $98.24 \pm 0.51$ & $93.53 \pm 0.89$ & $93.47 \pm 0.33$ \\
    16 & $100.00 \pm 0.00$ & $100.00 \pm 0.00$ & $100.00 \pm 0.00$ & $100.00 \pm 0.00$ & $100.00 \pm 0.00$ & $100.00 \pm 0.00$ & $100.00 \pm 0.01$ & $100.00 \pm 0.00$ \\
    32 & $100.00 \pm 0.00$ & $100.00 \pm 0.00$ & $100.00 \pm 0.00$ & $100.00 \pm 0.00$ & $100.00 \pm 0.00$ & $100.00 \pm 0.00$ & $100.00 \pm 0.00$ & $100.00 \pm 0.00$ \\
    64 & $100.00 \pm 0.00$ & $100.00 \pm 0.00$ & $100.00 \pm 0.00$ & $100.00 \pm 0.00$ & $100.00 \pm 0.00$ & $100.00 \pm 0.00$ & $100.00 \pm 0.00$ & $100.00 \pm 0.00$ \\
    96 & $100.00 \pm 0.00$ & $100.00 \pm 0.00$ & $100.00 \pm 0.00$ & $100.00 \pm 0.00$ & $100.00 \pm 0.00$ & $100.00 \pm 0.00$ & $100.00 \pm 0.00$ & $100.00 \pm 0.00$ \\
    \bottomrule
  \end{tabular}%
  }
\end{table*}

\begin{table*}[t]
  \centering
  \small
  \caption{Task accuracy (\%) for CNN in Setting 2 at shared-encoder depth $6$. Each entry is mean $\pm$ sample standard deviation over five random seeds.}
  \label{tab:cnn_train_acc_d6}
  \setlength{\tabcolsep}{3pt}
  \renewcommand{\arraystretch}{1.15}
  \resizebox{\textwidth}{!}{%
  \begin{tabular}{ccccccccc}
    \toprule
    \multirow{2}{*}{Channel width} & \multicolumn{2}{c}{$A_1B_1$} & \multicolumn{2}{c}{$A_1B_2$} & \multicolumn{2}{c}{$A_2B_1$} & \multicolumn{2}{c}{$A_2B_2$} \\
    \cmidrule(lr){2-3}\cmidrule(lr){4-5}\cmidrule(lr){6-7}\cmidrule(lr){8-9}
      & $A_1$ & $B_1$ & $A_1$ & $B_2$ & $A_2$ & $B_1$ & $A_2$ & $B_2$ \\
    \midrule
    1 & $95.73 \pm 0.92$ & $98.33 \pm 0.63$ & $90.42 \pm 0.63$ & $51.35 \pm 9.85$ & $59.09 \pm 4.89$ & $94.13 \pm 1.14$ & $53.93 \pm 1.51$ & $54.85 \pm 3.22$ \\
    2 & $99.85 \pm 0.09$ & $99.81 \pm 0.10$ & $96.75 \pm 1.09$ & $91.22 \pm 2.23$ & $90.58 \pm 1.47$ & $98.04 \pm 1.07$ & $86.73 \pm 2.39$ & $87.66 \pm 2.47$ \\
    4 & $100.00 \pm 0.00$ & $100.00 \pm 0.00$ & $99.95 \pm 0.04$ & $100.00 \pm 0.00$ & $100.00 \pm 0.00$ & $99.89 \pm 0.06$ & $99.75 \pm 0.11$ & $99.77 \pm 0.18$ \\
    8 & $100.00 \pm 0.00$ & $100.00 \pm 0.00$ & $100.00 \pm 0.00$ & $100.00 \pm 0.00$ & $100.00 \pm 0.00$ & $100.00 \pm 0.00$ & $100.00 \pm 0.00$ & $100.00 \pm 0.00$ \\
    16 & $100.00 \pm 0.00$ & $100.00 \pm 0.00$ & $100.00 \pm 0.00$ & $100.00 \pm 0.00$ & $100.00 \pm 0.00$ & $100.00 \pm 0.00$ & $100.00 \pm 0.00$ & $100.00 \pm 0.00$ \\
    32 & $100.00 \pm 0.00$ & $100.00 \pm 0.00$ & $100.00 \pm 0.00$ & $100.00 \pm 0.00$ & $100.00 \pm 0.00$ & $100.00 \pm 0.00$ & $100.00 \pm 0.00$ & $100.00 \pm 0.00$ \\
    64 & $100.00 \pm 0.00$ & $100.00 \pm 0.00$ & $100.00 \pm 0.00$ & $100.00 \pm 0.00$ & $100.00 \pm 0.00$ & $100.00 \pm 0.00$ & $100.00 \pm 0.00$ & $100.00 \pm 0.00$ \\
    96 & $100.00 \pm 0.00$ & $100.00 \pm 0.00$ & $100.00 \pm 0.00$ & $100.00 \pm 0.00$ & $100.00 \pm 0.00$ & $100.00 \pm 0.00$ & $100.00 \pm 0.00$ & $100.00 \pm 0.00$ \\
    \bottomrule
  \end{tabular}%
  }
\end{table*}

\begin{table*}[t]
  \centering
  \small
  \caption{Task accuracy (\%) for CNN in Setting 2 at shared-encoder depth $7$. Each entry is mean $\pm$ sample standard deviation over five random seeds.}
  \label{tab:cnn_train_acc_d7}
  \setlength{\tabcolsep}{3pt}
  \renewcommand{\arraystretch}{1.15}
  \resizebox{\textwidth}{!}{%
  \begin{tabular}{ccccccccc}
    \toprule
    \multirow{2}{*}{Channel width} & \multicolumn{2}{c}{$A_1B_1$} & \multicolumn{2}{c}{$A_1B_2$} & \multicolumn{2}{c}{$A_2B_1$} & \multicolumn{2}{c}{$A_2B_2$} \\
    \cmidrule(lr){2-3}\cmidrule(lr){4-5}\cmidrule(lr){6-7}\cmidrule(lr){8-9}
      & $A_1$ & $B_1$ & $A_1$ & $B_2$ & $A_2$ & $B_1$ & $A_2$ & $B_2$ \\
    \midrule
    1 & $99.17 \pm 0.31$ & $99.68 \pm 0.11$ & $93.61 \pm 0.41$ & $76.49 \pm 4.31$ & $76.18 \pm 2.57$ & $98.76 \pm 0.46$ & $69.54 \pm 4.68$ & $71.84 \pm 2.87$ \\
    2 & $100.00 \pm 0.00$ & $100.00 \pm 0.00$ & $99.92 \pm 0.12$ & $99.95 \pm 0.07$ & $99.92 \pm 0.03$ & $99.83 \pm 0.08$ & $99.24 \pm 0.36$ & $99.24 \pm 0.40$ \\
    4 & $100.00 \pm 0.00$ & $100.00 \pm 0.00$ & $100.00 \pm 0.00$ & $100.00 \pm 0.00$ & $100.00 \pm 0.00$ & $100.00 \pm 0.00$ & $100.00 \pm 0.00$ & $100.00 \pm 0.00$ \\
    8 & $100.00 \pm 0.00$ & $100.00 \pm 0.00$ & $100.00 \pm 0.00$ & $100.00 \pm 0.00$ & $100.00 \pm 0.00$ & $100.00 \pm 0.00$ & $100.00 \pm 0.00$ & $100.00 \pm 0.00$ \\
    16 & $100.00 \pm 0.00$ & $100.00 \pm 0.00$ & $100.00 \pm 0.00$ & $100.00 \pm 0.00$ & $100.00 \pm 0.00$ & $100.00 \pm 0.00$ & $100.00 \pm 0.00$ & $100.00 \pm 0.00$ \\
    32 & $100.00 \pm 0.00$ & $100.00 \pm 0.00$ & $100.00 \pm 0.00$ & $100.00 \pm 0.00$ & $100.00 \pm 0.00$ & $100.00 \pm 0.00$ & $100.00 \pm 0.00$ & $100.00 \pm 0.00$ \\
    64 & $100.00 \pm 0.00$ & $100.00 \pm 0.00$ & $100.00 \pm 0.00$ & $100.00 \pm 0.00$ & $100.00 \pm 0.00$ & $100.00 \pm 0.00$ & $100.00 \pm 0.00$ & $100.00 \pm 0.00$ \\
    96 & $100.00 \pm 0.00$ & $100.00 \pm 0.00$ & $100.00 \pm 0.00$ & $100.00 \pm 0.00$ & $100.00 \pm 0.00$ & $100.00 \pm 0.00$ & $100.00 \pm 0.00$ & $100.00 \pm 0.00$ \\
    \bottomrule
  \end{tabular}%
  }
\end{table*}

\clearpage

\onecolumn

\section{Detailed Language Task Accuracies}
\label{app:language_task_acc}

Tables~\ref{tab:setting2_task_acc}--\ref{tab:setting4_task_acc} report the accuracy of each task separately in the four task-pair contexts across LoRA ranks $r$. Each context has two columns for its active tasks $A_i$ and $B_j$; the task assignments follow Table~\ref{tab:task_combinations}. Values are percentages, reported as mean $\pm$ sample standard deviation over three random seeds (denominator $3-1$), rounded to two decimal places. The seeds are $1,2,3$ for Setting 2, $42,43,44$ for Setting 3, and $42,123,2026$ for Setting 4. Setting 3 uses the final-epoch training accuracies corresponding to the reported pairwise correlations.

\begin{table}[!htbp]
  \centering
  \small
  \caption{Task accuracy (\%) for BERT in Setting 2: $A_1$ = PAWS-X-en, $A_2$ = PAWS-X-fr, $B_1$ = PAWS-X-ja, and $B_2$ = PAWS-X-ko. Each entry is mean $\pm$ sample standard deviation over three random seeds.}
  \label{tab:setting2_task_acc}
  \setlength{\tabcolsep}{3pt}
  \renewcommand{\arraystretch}{1.15}
  \resizebox{\textwidth}{!}{%
  \begin{tabular}{ccccccccc}
    \toprule
    \multirow{2}{*}{LoRA rank $r$} & \multicolumn{2}{c}{$A_1B_1$} & \multicolumn{2}{c}{$A_1B_2$} & \multicolumn{2}{c}{$A_2B_1$} & \multicolumn{2}{c}{$A_2B_2$} \\
    \cmidrule(lr){2-3}\cmidrule(lr){4-5}\cmidrule(lr){6-7}\cmidrule(lr){8-9}
      & $A_1$ & $B_1$ & $A_1$ & $B_2$ & $A_2$ & $B_1$ & $A_2$ & $B_2$ \\
    \midrule
    1 & $95.32 \pm 0.14$ & $90.85 \pm 0.85$ & $95.52 \pm 0.18$ & $89.15 \pm 0.96$ & $93.44 \pm 0.21$ & $91.50 \pm 0.40$ & $94.29 \pm 0.43$ & $89.98 \pm 0.92$ \\
    2 & $97.24 \pm 0.37$ & $94.43 \pm 0.45$ & $97.42 \pm 0.14$ & $94.46 \pm 0.05$ & $96.22 \pm 0.08$ & $95.24 \pm 0.60$ & $96.25 \pm 0.08$ & $93.86 \pm 0.27$ \\
    4 & $98.80 \pm 0.11$ & $97.88 \pm 0.24$ & $98.74 \pm 0.06$ & $97.36 \pm 0.19$ & $98.01 \pm 0.22$ & $97.52 \pm 0.20$ & $98.05 \pm 0.11$ & $96.89 \pm 0.34$ \\
    8 & $99.51 \pm 0.03$ & $98.87 \pm 0.08$ & $99.27 \pm 0.05$ & $98.39 \pm 0.01$ & $98.92 \pm 0.06$ & $98.83 \pm 0.07$ & $98.99 \pm 0.11$ & $98.54 \pm 0.03$ \\
    16 & $99.52 \pm 0.04$ & $98.88 \pm 0.08$ & $99.38 \pm 0.07$ & $98.73 \pm 0.05$ & $98.98 \pm 0.04$ & $98.89 \pm 0.05$ & $98.93 \pm 0.02$ & $98.53 \pm 0.10$ \\
    32 & $99.92 \pm 0.02$ & $99.72 \pm 0.05$ & $99.97 \pm 0.02$ & $99.66 \pm 0.03$ & $99.74 \pm 0.08$ & $99.61 \pm 0.01$ & $99.63 \pm 0.04$ & $99.72 \pm 0.03$ \\
    64 & $99.97 \pm 0.01$ & $99.86 \pm 0.02$ & $99.99 \pm 0.01$ & $99.83 \pm 0.01$ & $99.87 \pm 0.04$ & $99.81 \pm 0.01$ & $99.82 \pm 0.02$ & $99.87 \pm 0.01$ \\
    \bottomrule
  \end{tabular}%
  }
\end{table}

\begin{table}[!htbp]
  \centering
  \small
  \caption{Final-epoch training accuracy (\%) for BERT in Setting 3: $A_1$ = SST-2, $A_2$ = CommonsenseQA, $B_1$ = MRPC, and $B_2$ = MathQA. Each entry is mean $\pm$ sample standard deviation over three random seeds.}
  \label{tab:setting3_task_acc}
  \setlength{\tabcolsep}{3pt}
  \renewcommand{\arraystretch}{1.15}
  \resizebox{\textwidth}{!}{%
  \begin{tabular}{ccccccccc}
    \toprule
    \multirow{2}{*}{LoRA rank $r$} & \multicolumn{2}{c}{$A_1B_1$} & \multicolumn{2}{c}{$A_1B_2$} & \multicolumn{2}{c}{$A_2B_1$} & \multicolumn{2}{c}{$A_2B_2$} \\
    \cmidrule(lr){2-3}\cmidrule(lr){4-5}\cmidrule(lr){6-7}\cmidrule(lr){8-9}
      & $A_1$ & $B_1$ & $A_1$ & $B_2$ & $A_2$ & $B_1$ & $A_2$ & $B_2$ \\
    \midrule
    1 & $99.82 \pm 0.08$ & $99.81 \pm 0.15$ & $99.41 \pm 0.10$ & $63.37 \pm 2.46$ & $90.26 \pm 0.64$ & $99.19 \pm 0.14$ & $90.79 \pm 0.63$ & $45.68 \pm 0.81$ \\
    2 & $99.96 \pm 0.02$ & $99.96 \pm 0.04$ & $99.66 \pm 0.08$ & $82.93 \pm 2.46$ & $97.41 \pm 0.49$ & $99.80 \pm 0.09$ & $95.63 \pm 0.51$ & $70.18 \pm 2.14$ \\
    4 & $99.99 \pm 0.02$ & $99.98 \pm 0.04$ & $99.83 \pm 0.10$ & $95.57 \pm 0.29$ & $99.22 \pm 0.18$ & $99.93 \pm 0.03$ & $97.70 \pm 0.38$ & $91.53 \pm 0.59$ \\
    8 & $99.99 \pm 0.02$ & $99.98 \pm 0.03$ & $99.90 \pm 0.00$ & $98.76 \pm 0.10$ & $99.68 \pm 0.08$ & $99.98 \pm 0.02$ & $99.30 \pm 0.22$ & $97.89 \pm 0.22$ \\
    16 & $99.99 \pm 0.02$ & $99.99 \pm 0.02$ & $99.99 \pm 0.02$ & $99.46 \pm 0.13$ & $99.68 \pm 0.14$ & $99.98 \pm 0.02$ & $99.61 \pm 0.17$ & $99.44 \pm 0.05$ \\
    32 & $99.99 \pm 0.02$ & $99.99 \pm 0.01$ & $99.96 \pm 0.08$ & $98.94 \pm 0.67$ & $99.81 \pm 0.02$ & $99.98 \pm 0.04$ & $99.82 \pm 0.11$ & $99.48 \pm 0.25$ \\
    64 & $99.99 \pm 0.02$ & $100.00 \pm 0.00$ & $99.98 \pm 0.04$ & $99.60 \pm 0.03$ & $99.74 \pm 0.10$ & $99.96 \pm 0.02$ & $99.79 \pm 0.11$ & $99.67 \pm 0.12$ \\
    \bottomrule
  \end{tabular}%
  }
\end{table}

\begin{table}[!htbp]
  \centering
  \small
  \caption{Task accuracy (\%) for Qwen3-8B in Setting 4: $A_1$ = MathQA, $A_2$ = MATH, $B_1$ = RACE-Middle, and $B_2$ = LogiQA2.0. Each entry is mean $\pm$ sample standard deviation over three random seeds.}
  \label{tab:setting4_task_acc}
  \setlength{\tabcolsep}{3pt}
  \renewcommand{\arraystretch}{1.15}
  \resizebox{\textwidth}{!}{%
  \begin{tabular}{ccccccccc}
    \toprule
    \multirow{2}{*}{LoRA rank $r$} & \multicolumn{2}{c}{$A_1B_1$} & \multicolumn{2}{c}{$A_1B_2$} & \multicolumn{2}{c}{$A_2B_1$} & \multicolumn{2}{c}{$A_2B_2$} \\
    \cmidrule(lr){2-3}\cmidrule(lr){4-5}\cmidrule(lr){6-7}\cmidrule(lr){8-9}
      & $A_1$ & $B_1$ & $A_1$ & $B_2$ & $A_2$ & $B_1$ & $A_2$ & $B_2$ \\
    \midrule
    1 & $98.27 \pm 0.09$ & $98.89 \pm 0.12$ & $97.48 \pm 0.17$ & $97.58 \pm 0.08$ & $67.79 \pm 0.45$ & $99.00 \pm 0.06$ & $61.76 \pm 0.33$ & $99.07 \pm 0.03$ \\
    2 & $99.92 \pm 0.02$ & $99.52 \pm 0.12$ & $99.78 \pm 0.02$ & $99.02 \pm 0.14$ & $83.39 \pm 0.45$ & $99.79 \pm 0.05$ & $79.22 \pm 0.29$ & $99.71 \pm 0.07$ \\
    4 & $99.90 \pm 0.03$ & $99.63 \pm 0.10$ & $99.94 \pm 0.02$ & $99.48 \pm 0.13$ & $91.57 \pm 0.28$ & $99.92 \pm 0.05$ & $87.84 \pm 0.27$ & $99.97 \pm 0.03$ \\
    8 & $99.98 \pm 0.02$ & $99.91 \pm 0.07$ & $99.98 \pm 0.02$ & $99.74 \pm 0.07$ & $96.52 \pm 0.25$ & $99.88 \pm 0.02$ & $95.72 \pm 0.25$ & $99.87 \pm 0.00$ \\
    16 & $100.00 \pm 0.00$ & $99.99 \pm 0.02$ & $99.99 \pm 0.02$ & $99.86 \pm 0.05$ & $99.10 \pm 0.17$ & $99.97 \pm 0.00$ & $99.22 \pm 0.07$ & $99.96 \pm 0.02$ \\
    32 & $100.00 \pm 0.00$ & $99.99 \pm 0.02$ & $100.00 \pm 0.00$ & $99.99 \pm 0.02$ & $99.94 \pm 0.05$ & $100.00 \pm 0.00$ & $99.93 \pm 0.03$ & $99.98 \pm 0.02$ \\
    64 & $100.00 \pm 0.00$ & $99.99 \pm 0.02$ & $100.00 \pm 0.00$ & $99.98 \pm 0.02$ & $99.99 \pm 0.02$ & $100.00 \pm 0.00$ & $99.97 \pm 0.03$ & $99.99 \pm 0.02$ \\
    \bottomrule
  \end{tabular}%
  }
\end{table}

\twocolumn

\bibliographystyle{named}
\bibliography{ijcai26}

\end{document}